\documentclass[11pt]{article}

\pdfoutput=1  

 \usepackage{url}                
 \usepackage{booktabs}       
 \usepackage{amsfonts}       
 \usepackage{nicefrac}       
 \usepackage{microtype}      

\usepackage{fncylab,enumerate,wrapfig,placeins}
\usepackage{bbm}
\usepackage{graphicx}
\usepackage{enumitem}  
\usepackage{amsmath,amsthm,amsfonts,amssymb}  
\usepackage{pifont}

 \usepackage{relsize} 
\usepackage{accents}  

\usepackage{caption}
 \usepackage{textgreek,upgreek,bm}

 \usepackage{titlesec}

\usepackage{stackengine}

\usepackage{geometry}

\pdfoutput=1  

\usepackage{caption}
\usepackage{booktabs}

\usepackage{multirow}

\usepackage{algorithm}
\usepackage{algpseudocode}

\algnewcommand\algorithmicinput{\textbf{Input:}}
\algnewcommand\INPUT{\item[\algorithmicinput]}
\algnewcommand\algorithmicoutput{\textbf{Output:}}
\algnewcommand\OUTPUT{\item[\algorithmicoutput]}

\usepackage[usenames,dvipsnames]{color}
\usepackage{color}

\usepackage[colorlinks,%
linkcolor=BrickRed,%
filecolor =BrickRed,%
urlcolor=blue,%
citecolor=RoyalPurple,%
]{hyperref}

\usepackage[nameinlink,noabbrev]{cleveref}

\Crefname{corollary}{Corollary}{Corollaries}
\Crefname{eqnarray}{eq.}{eqs.}
\Crefname{equation}{eq.}{eqs.}

\Crefname{figure}{Fig.}{Figs.}
\Crefname{tabular}{Tab.}{Tabs.}
\Crefname{table}{Tab.}{Tabs.}
\Crefname{lemma}{Lemma}{Lemmas}
\Crefname{proposition}{Prop.}{Propositions}
\Crefname{theorem}{Thm.}{Thms.}
\Crefname{definition}{Def.}{Defs.} 
\Crefname{section}{Section}{Sections}
\Crefname{assumption}{Assumption}{Assumptions}
\Crefname{exmp}{Example}{Examples}
\Crefname{exercise}{Exercise}{Exercises}

 \usepackage{textgreek,upgreek,bm}

\def\BelErr{\mathcal{B}}

\def\lambdamin{\lambda_{\textup{\rm\tiny min}}}
\def\lambdamax{\lambda_{\textup{\rm\tiny max}}}

\def\EXP{\mathcal{W}}

\def\tEXP{{\text{\tiny$\EXP$}}}

\def\tTheta{{\text{\tiny$\Theta$}}}

 \def\feeUnique{\mathcal{C}^\tTheta}

 \def\ctheta{c^\tTheta}
\def\cEXP{c^\tEXP}

 \def\Rtheta{R^\tTheta}

\def\REXP{R^\tEXP}

 \def\bEXP{b^\tEXP} 
 \def\btheta{b^\tTheta}

\graphicspath{{figures/}}

\def\Fig#1{Fig.~\ref{#1}}

\newlength{\noteWidth}
\long\def\notes#1{\ifinner
{\footnotesize #1}
\else 
\marginpar{\parbox[t]{\noteWidth}{\raggedright\footnotesize#1}}
\fi\typeout{#1}}

\def\notes#1{}

\def\urls#1{{\footnotesize\url{#1}}}

\def\mindex#1{\index{#1}}

\def\ocp{*}   

\DeclareFontFamily{U}{mathx}{\hyphenchar\font45}
\DeclareFontShape{U}{mathx}{m}{n}{<-> mathx10}{}
\DeclareSymbolFont{mathx}{U}{mathx}{m}{n}
\DeclareMathAccent{\widebar}{0}{mathx}{"73}

\def\BE{{\cal B}}
\def\maxBE{\overline{\cal B}}

\def\Tdiff{\mathcal{D}}

\def\tilTdiff{\widetilde{\mathcal{D}}}

\def\thetaPR{\theta^{\text{\tiny\sf  PR}}}
\def\tilthetaPR{\tilde{\theta}^{\text{\tiny\sf  PR}}}
\def\SigmaPR{\Sigma^{\text{\tiny\sf PR}}}

\def\odestate{\upvartheta}

\def\fee{\upphi}
\def\feex{\widetilde{\fee}}

\def\elig{\zeta}

\def\uH{\underline{H}}
\def\uQ{\underline{Q}}

\def\disc{\gamma}

\def\stepf{\beta}

\newcommand{\bbblot}{\raise1pt\hbox{\vrule height .4ex width .4ex depth .05ex}}

\long\def\defbox#1{\framebox[.9\hsize][c]{\parbox{.85\hsize}{%
\parindent=0pt
\baselineskip=12pt plus .1pt      
\parskip=6pt plus 1.5pt minus 1pt 
 #1}}}

\long\def\beginbox#1\endbox{\subsection*{}%
\hbox{\hspace{.05\hsize}\defbox{\medskip#1\bigskip}}%
\subsection*{}}

\def\endbox{}

 \def\archival#1{} 

\def\FRAC#1#2#3{\genfrac{}{}{}{#1}{#2}{#3}}

\def\ddt{{\mathchoice{\FRAC{1}{d}{dt}}%
{\FRAC{1}{d}{dt}}%
{\FRAC{3}{d}{dt}}%
{\FRAC{3}{d}{dt}}}}

\def\ddtp{{\mathchoice{\FRAC{1}{d^{\hbox to 2pt{\rm\tiny +\hss}}}{dt}}%
{\FRAC{1}{d^{\hbox to 2pt{\rm\tiny +\hss}}}{dt}}%
{\FRAC{3}{d^{\hbox to 2pt{\rm\tiny +\hss}}}{dt}}%
{\FRAC{3}{d^{\hbox to 2pt{\rm\tiny +\hss}}}{dt}}}}

\def\ddyp{{\mathchoice{\FRAC{1}{d^{\hbox to 2pt{\rm\tiny +\hss}}}{dy}}%
{\FRAC{1}{d^{\hbox to 2pt{\rm\tiny +\hss}}}{dy}}%
{\FRAC{3}{d^{\hbox to 2pt{\rm\tiny +\hss}}}{dy}}%
{\FRAC{3}{d^{\hbox to 2pt{\rm\tiny +\hss}}}{dy}}}}

\def\half{{\mathchoice{\FRAC{1}{1}{2}}%
{\FRAC{1}{1}{2}}%
{\FRAC{3}{1}{2}}%
{\FRAC{3}{1}{2}}}}

\def\det{{\mathop{\rm det}}}
\def\limsup{\mathop{\rm lim{\,}sup}}

\def\argmin{\mathop{\rm arg{\,}min}}

\def\state{{\sf X}}

\def\ustate{{\sf U}}

\def\bfmath#1{{\mathchoice{\mbox{\boldmath$#1$}}%
{\mbox{\boldmath$#1$}}%
{\mbox{\boldmath$\scriptstyle#1$}}%
{\mbox{\boldmath$\scriptscriptstyle#1$}}}}

\def\bfPhi{\bfmath{\Phi}}

\def\bfmI{\bfmath{I}}

\def\bfmX{\bfmath{X}}

\def\bfmY{\bfmath{Y}}

\def\bfmhhaY{\bfmath{\hhaY}} 
\def\bfmhhaY{\hbox to 0pt{$\widehat{\bfmY}$\hss}\widehat{\phantom{\raise 1.25pt\hbox{$\bfmY$}}}}

\def\haf{{\hat f}}

\def\haA{\widehat A}

\def\tiltheta{{\tilde \theta}}

\def\clF{{\cal F}}

\def\clU{{\cal U}}

\def\clZ{{\cal Z}}

\def\tilq{\tilde{q}}

\def\eqdef{\mathbin{:=}}

\def\Prob{{\sf P}}

\def\Expect{{\sf E}}

\def\lgmath#1{{\mathchoice{\mbox{\large #1}}%
{\mbox{\large #1}}%
{\mbox{\tiny #1}}%
{\mbox{\tiny #1}}}}

\def\Zero{{\mathchoice{\lgmath{\sf 0}}%
{\mbox{\sf 0}}%
{\mbox{\tiny \sf 0}}%
{\mbox{\tiny \sf 0}}}}

\def\ind{\bbbone}
  
 \def\epsy{\varepsilon}

\def\varble{\,\cdot\,}

\def\formtmp#1#2{{\vskip12pt\noindent\fboxsep=0pt\colorbox{#1}{\vbox{\vskip3pt\hbox to \textwidth{\hskip3pt\vbox{\raggedright\noindent\textbf{#2\vphantom{Qy}}}\hfill}\vspace*{3pt}}}\par\vskip2pt%
\noindent\kern0pt}}

\def\psisub#1{\psi_{(#1)}}

\def\csub#1{c_{#1}}

\def\barb{{\overline {b}}}

\def\barf{{\widebar{f}}}

\def\barf{{\widebar{f}}}

{\end{list}}

\def\ass(#1:#2){(#1\ref{#1:#2})}

\def\ritem#1{
\item[{\sf \ass(\current_model:#1)}]
}

\newenvironment{recall-ass}[1]{%
\begin{description}
\def\current_model{#1}}{
\end{description}
}

\def\sq{\hbox{\rlap{$\sqcap$}$\sqcup$}}
\def\qed{\ifmmode\sq\else{\unskip\nobreak\hfil
\penalty50\hskip1em\null\nobreak\hfil\sq
\parfillskip=0pt\finalhyphendemerits=0\endgraf}\fi}

\newcommand{\blot}{\vrule height 1.1ex width .9ex depth -.1ex }
\def\qedb{\ifmmode\blot\else{\vspace{-.2cm}\unskip\nobreak\hfil
\penalty50\hskip1em\null\nobreak\hfil\blot
\parfillskip=0pt\finalhyphendemerits=0\endgraf}\fi}

\newcounter{rmnum}
\newenvironment{romannum}{\begin{list}{{\upshape (\roman{rmnum})}}{\usecounter{rmnum}
\setlength{\leftmargin}{18pt}
\setlength{\rightmargin}{8pt}
\setlength{\itemindent}{2pt}
}}{\end{list}}

\newcounter{anum}

\newcommand{\field}[1]{\mathbb{#1}}

\def\Re{\field{R}} 

\def\intgr{\field{Z}}

\def\Co{\field{C}}

\def\Prob{{\sf P}}

\def\Expect{{\sf E}}

\def\transpose{{\intercal}}

\def\argmin{\mathop{\rm arg\, min}}
\def\ind{\hbox{\large \bf 1}}

\def\trace{\hbox{\rm trace\,}}  
\def\epsy{\varepsilon}
\def\varble{\,\cdot\,}

\def\haY{\widehat{Y}}

\def\hhaY{\hbox to 0pt{$\haY$\hss}\widehat{\phantom{\raise 1.25pt\hbox{Y}}}}

\def\haA{\widehat A}

\def\haY{\widehat Y}

\def\bfPhi{\bfmath{\Phi}}

\newlength{\dhatheight}

\newtheorem{theorem}{Theorem}[section]

\newtheorem{proposition}[theorem]{Proposition}
\newtheorem{lemma}[theorem]{Lemma}

\newsavebox{\junk}
\savebox{\junk}[1.6mm]{\hbox{$|\!|\!|$}}

\def\det{{\mathop{\rm det}}}
\def\limsup{\mathop{\rm lim\ sup}}

\def\argmin{\mathop{\rm arg\, min}}

\def\clF{{\cal F}}

\def\clU{{\cal U}}

\def\clZ{{\cal Z}}

\makeatletter
\newcommand\gobblepars{%
    \@ifnextchar\par%
        {\expandafter\gobblepars\@gobble}%
{}}
\makeatother

\def\whamrm#1{\smallbreak\pagebreak[3]%
	\noindent\text{\rm#1}\ \ \gobblepars}
	
\def\whamit#1{\smallbreak\pagebreak[3]%
	\noindent\textit{#1}\ \ \gobblepars}

\def\wham#1{\smallbreak\pagebreak[3]%
	\noindent\textbf{#1}\ \ \gobblepars}

\def\uE{\underline{E}}

\def\barfzap{\bar{f}^{\textup{\textsf{zap}}}}

\def\bdd#1{b^{\text{\rm\tiny\ref{#1}}}}

\def\Fig#1{Fig.~\ref{#1}}

\def\state{{\sf X}}
\def\eqdef{\mathbin{:=}}
\def\trace{\hbox{\rm trace\,}}

\def\ind{\field{I}}

\def\Re{\field{R}}

\def\ocp{*}   
\def\fee{\upphi}

\def\barb{\bar b}

\def\odestate{\upvartheta}

\def\wham#1{\smallbreak\pagebreak[3]%
	\noindent\textup{\textbf{#1}}\ \ \gobblepars}
	
\def\whamrm#1{\smallbreak\pagebreak[3]%
	\noindent\textup{\text{\rm#1}}\ \ \gobblepars}

\def\whamit#1{\smallbreak\pagebreak[3]%
	\noindent\textit{#1}\ \ \gobblepars}

\def\barfinf{\barf_{\infty}}

\def\SigmaTheta{\Sigma_\tTheta}

\def\thetaPR{\theta^{\text{\tiny\sf  PR}}}
\def\tilthetaPR{\tilde{\theta}^{\text{\tiny\sf  PR}}}
\def\SigmaPR{\Sigma^{\text{\tiny\sf PR}}_\tTheta}

\title{Stability of Q-Learning    Through Design and Optimism
}

\author{Sean Meyn
\thanks{SPM is with the University of Florida, Gainesville, FL 32611.
	Financial support from ARO award W911NF2010055 
		is gratefully acknowledged.  This article was created in part to support the 2023 INFORMS Applied Probability Society lecture
		--- slides available at {\tt researchgate.net} \cite{APS2023}}%
}

\begin{document}

\maketitle

\begin{abstract} 

     Q-learning  has   become an important part of the reinforcement learning toolkit since its introduction in the dissertation of
     Chris Watkins in the 1980s.    
        The purpose of this paper is in part a tutorial on stochastic approximation and Q-learning, providing details regarding the   INFORMS APS 2023,  inaugural Applied Probability Trust Plenary Lecture,  presented in  Nancy France, June 2023. 

The paper also presents   new approaches to ensure stability and potentially accelerated convergence for these algorithms, and stochastic approximation in other settings.    Two contributions are entirely new:

\wham{1.}     
Stability of Q-learning with linear function approximation has been an open topic for research for over three decades.  It is shown that with appropriate optimistic training in the form of a modified Gibbs policy,  there exists a solution to the projected Bellman equation, and   the algorithm is stable (in terms of bounded parameter estimates).   Convergence remains one of many open topics for research.

\wham{2.}   
The new Zap Zero algorithm is designed to approximate the Newton-Raphson flow without matrix inversion.    It is stable and convergent under mild assumptions on the mean flow vector field for the algorithm, and compatible statistical assumption on an underlying Markov chain.    The algorithm is a general approach to stochastic approximation which in particular applies to Q-learning with ``oblivious'' training even with non-linear function approximation.

 \medbreak
 

\end{abstract}

\newgeometry{top=1in,bottom=1in,right=1in,left=1in}


\section{Introduction}

The article concerns Q-learning algorithms, motivated by the same objective as in the first formulation of Watkins \cite{watday92a,wat89}: the infinite-horizon optimal control problem,   with state-action value function  
\begin{equation}
Q^\star(x,u) = \min \sum_{k=0}^\infty \disc^k \Expect [ c(X_k,U_k) \mid X_0 =x\,, \ U_0=u]
\label{e:Q}
\end{equation}
The state process $\{ X_k  : k\ge 0\}$ evolves on a finite state space denoted $\state$,   and the
action (or input)  process $\{  U_k  : k\ge 0\}$ evolves on a finite set $\ustate$; 
$c\colon\state\times\ustate\to\Re$ is the one-step reward function,  and $\disc\in(0,1)$ the discount factor.   

The minimum in \eqref{e:Q} is over all history dependent input sequences.   
Under standard Markovian assumptions reviewed in \Cref{s:SARL},  an optimal input is obtained by state feedback $\fee^\star\colon\state\to\ustate$,  with $\fee^\star (x) \in \argmin Q^\star(x,u) $ for each $x$  \cite{ber12a}.   Moreover, 
the \textit{Q-function} $Q^\star$ solves the Bellman equation,
\begin{equation}
Q^\star(x,u) =  c(x,u) + \disc   \Expect[  \uQ^\star (X_{k+1}) \mid X_k=x\,, \, U_k =u] ,,\qquad x \in\state\,,\ u\in\ustate\,,\ k\ge 0
\label{e:DCOE-Q}
\end{equation}
where throughout the paper an under-bar denotes a minimum:  $\uH(x) \eqdef\min_u H(x,u)$, $x\in\state$, for any function $H\colon\state\times\ustate\to\Re$.

The objective of Q-learning is to obtain an approximate solution to \eqref{e:DCOE-Q} among a parameterized class $\{ Q^\theta :  \theta\in\Re^d \}$.  Typical in theoretical analysis is linear function approximation,    
  $\{Q^\theta = \theta^\transpose \psi : \theta\in\Re^d\}$ with $\psi$ a vector of basis functions.
   
 Given an approximation within this class, we obtain a policy  (i.e.\ state feedback law) 
  $\fee^\theta\colon\ustate\to\state$:
  \begin{equation}
\fee^\theta(x) \in \argmin_u  Q^\theta(x,u)
\label{e:feetheta}
\end{equation}
with some fixed rule in place in case of ties.
\begin{subequations}%
 Much of the present article focuses on a generalization of the original algorithm of Watkins:
For initialization $\theta_0\in\Re^d  $, define the sequence of estimates  recursively:  
\begin{align}
\theta_{n+1} &= \theta_n + \alpha_{n+1}   \Tdiff_{n+1} \elig_n \,,
		 \qquad\qquad   \elig_n  = \nabla_\theta  Q^\theta(X_{n}, U_{n})  \big|_{\theta =\theta_n}
\label{e:Qlambda}
\\
\Tdiff_{n+1} & =  c(X_n,U_n)   + \disc     \uQ^{\theta_n} (X_{n+1})  - Q^{\theta_n}(X_n,U_n) \,.
\label{e:BE_theta_n}
\end{align} 
in which $\{
\alpha_n\}$ is a non-negative step-size sequence. 
See \cite{sutbar18,sze10,CSRL} for a range of interpretations of the algorithm.  The vectors $\{\elig_n  \}$ are entirely analogous to the   eligibility vectors  used in the TD($0$) algorithm \cite{sut88,tsivan97},  and    $\{\Tdiff_{n+1} \} $ is known as the temporal difference sequence.  
The recursion \eqref{e:Qlambda} reduces to the original tabular Q-learning algorithm  when using a tabular basis  \cite{watday92a,wat89} (see \Cref{s:tabQ}
 for   definitions).

 \label{e:Qintro}
\end{subequations}

The goal of Q-learning is to approximate the solution to the \textit{projected Bellman equation},
 \begin{equation}
 \Zero  =  \Expect\bigl[ \{   c(X_n,U_n)   + \disc     \uQ^{\theta^*} (X_{n+1})  - Q^{\theta^*}(X_n,U_n) \}  \elig_n \bigr]
\label{e:PBE}
\end{equation}
in which the expectation is in steady-state.

Soon after   Q-learning was introduced, it was recognized that the algorithm can be cast within the framework of stochastic approximation (SA)  \cite{tsi94a,jaajorsin94a}.    To explain the contributions and approach to analysis in this paper it is necessary to first explain why  \eqref{e:Qlambda} can also be cast as an SA recursion, subject to mild assumptions on the input used for training.

\subsection{A few warnings}

For readers with background in reinforcement learning, some notation may not be familiar, and some goals may not seem standard.  

  \wham{1.}   We use $\uppi$ for invariant measures, following a long tradition in the theory of Markov chains \cite{MT}.    
  Apologies to those of you who prefer ``pi'' for ``policy''.
 
\wham{2.} 
 Finite-$n$ bounds (sample complexity bounds) are valuable in the theory of bandits. 
 There has not been comparable success  in reinforcement learning, in part because present bounds are very loose.    Perhaps sample complexity theory will evolve to become more practical.   This paper focuses on asymptotic statistics for comparing algorithms,   
 as well as heuristics based on ODE techniques to gain insight on transient behavior.

The most valuable tool from asymptotic statistics is  the Central Limit Theorem. 
For the basic SA recursion \eqref{e:SAa},  the CLT typically holds for the scaled error  $z_n = \tiltheta_n /\sqrt{\alpha_n}$ with $ \tiltheta_n =\theta_n -\theta^*$,    along with convergence of the scaled mean-square error:
\[
\lim_{n\to\infty}  \frac{1}{\alpha_n}  \Expect[\tiltheta_n \tiltheta_n^\transpose ] =  \SigmaTheta
\]
We typically take $\{\alpha_n\}$ ``big'' to reduce transients, such as $\alpha_n = n^\rho$ with $\half <\rho < 1$.   The limit above implies slow convergence for large $n$, but this is ameliorated via the averaging technique of  Polyak and Ruppert yielding
\[
\lim_{n\to\infty}  n\Expect[\tilthetaPR_n {\tilthetaPR_n}^\transpose ] =  \SigmaPR
\] 
in which $\SigmaPR$ is minimal in a matricial sense---see \Cref{s:SAcrash}
for definitions.

This covariance matrix can be estimated using the batch means method,  which requires performing many relatively short runs with distinct initial conditions \cite{asmgly07}.

\subsection{Some history}
\label{s:prior}

One open issue motivating the research surveyed in this paper is this:   
 \textit{it is not   known if  the projected Bellman equation \eqref{e:PBE}  has a solution outside of very special cases}.  
 
Success stories surveyed in \cite{sze10} include the special case of binning  \cite{jaajorsin94a}, which is a generalization of the tabular setting,  and
   the criterion in  \cite{melmeyrib08} and its improvement in  \cite{leehe19},  but the assumptions are not easily verified in practice.  The progress report in \cite[Section 3.3.2]{sze10} states that the \textit{only known convergence result is due to Melo et al.} \cite{melmeyrib08}.  
See     \cite[Section 11.2]{sutbar18} for further discussion,   and  \cite{goptho22} for recent insight.  

 This open problem was a topic of discussion throughout the Simons program on reinforcement learning held in 2020, especially during the bootcamp lectures \cite{Simons_bootcamp2020}.
 
 \Cref{t:Qstable,t:greedyStable}    resolve this open problem for Q-learning with optimistic training. Following many preliminaries, the proof of  \Cref{t:Qstable}
 is similar to the proof of convergence of TD($\lambda$) learning from the dissertation of Van Roy \cite{tsivan97,royThesis98},  and the assumptions are related to the assumptions in this prior work, even though the setting is very different.

The recent paper \cite{limkimlee22}  considers Q-learning with linear function approximation and oblivious training.   
With sufficiently large regularization they obtain a unique equilibrium for the algorithm that approximates the solution to the projected Bellman equation.  
It is likely that their results can be improved using optimistic training as in the present work.

 Also recent is the work of \cite{checlamag22},  which is cast in a similar setting:   Q-learning with linear function approximation and oblivious training.   It is argued that the use of a target network combined with 
		a carefully constructed projection of parameters
improves performance,  and their error bounds are consistent with their claims.  
While the paper is a significant step forward, they leave open the question of existence of a solution to the projected Bellman equation.
With vanishing step-size, if convergence is established with or without a target network,  the limit must be a solution to the projected Bellman equation (see \cite[Proposition 5.10]{CSRL} for   proof in the case of deterministic optimal control---the arguments in the stochastic setting are identical).

The lack of theory motivated  Baird's gradient descent approach \cite{bai95}  (and his counterexample discussed in \Cref{s:BairdNumerics}),
 as well as GQ learning  \cite{maeszebhasut10},  in which the root finding problem is replaced with the minimization of a loss function.
See \cite{avrbordolpat21} for  recent theory. 

Zap stochastic approximation was introduced to ensure convergence, and also provide acceleration   \cite{devmey17a,devmey17b}.   While originally proposed for Q-learning with linear function approximation, it was later shown to be convergent even with nonlinear function approximation \cite{chedevbusmey20b}, and the general technique applies to any application in which stochastic approximation is used.  
A version of the Zap-Zero algorithm was introduced in \cite{CSRL},  whose form is motivated in part by the   time-scale SA algorithm introduced in  \cite{maeszebhasut10}.    

 The new Zap-Zero algorithm  \eqref{e:ZapZero} is entirely new, and convergent under far weaker conditions.   

Much recent research has focused on \textit{linear} MDPs, notably \cite{yanwan19,yanwan19b,jinyanwanjor20},  in which the system dynamics are partially known:   for a known ``feature map'' $\phi \colon\state\times\ustate \to \Re^d$  and an unknown sequence of probability measures $\{\mu_i : 1\le i\le d\}$ on $\state$,   a linear MDP is assumed to have a controlled transition matrix of the form $P_u(x,x')  =\sum \phi_i(x,u) \mu_i(x')$.   There is now a relatively complete theory  for this special case,  in which the algorithm is designed based on knowledge of the feature map.

The reader is encouraged to see   \cite{bascurkraneu21,mehmeyneulu21,lumehmeyneu22,margarlyg22,margardralyg22} 
for new approaches to Q-learning based on convex programming approaches to MDPs.      It is hoped that the analytic techniques presented in this paper may be adapted to these new algorithms.

\subsection{Overview}   Following a summary of notation and key results from stochastic approximation theory in \Cref{s:SARL}, the paper sets out to survey results from the theory of Q-learning, including these highlights:

\wham{1.}  \Cref{s:Q} reviews theory for Q-learning with linear function approximation.   It is now well known that there are challenges even in the simplest tabular setting, in which convergence holds but is very slow.  Methods are surveyed to accelerate convergence.   The theory is restricted to \textit{oblivious} training, meaning that the input during training is independent of the parameter estimates. 

Consideration of optimistic policies is postponed to \Cref{s:Qstable},  which contains entirely new theory:   if a smooth approximation of the $\epsy$-greedy policy is used for training, then under mild conditions the parameter estimates are bounded, and there exists a solution to the projected Bellman equation (see \Cref{t:Qstable}).
Unfortunately,  convergence to $\theta^*$ remains one of many open problems for research

\wham{2.}  \Cref{s:Z} contains a survey of the author's favorite approach known as \textit{Zap Q-learning};   the theory is elegant and the approach is stable even with nonlinear function approximation.    A major problem with this approach is the need for a  matrix inversion in each iteration of the algorithm
A  new algorithm and theory is presented here for the first time in \Cref{s:ZZ}:
the \textit{Zap Zero algorithm} is   designed to avoid matrix inversion,  
	  and complexity of matrix-vector multiplication   
can be tamed (see \Cref{t:ZapZero}).

The theory in   \Cref{s:Z,s:ZZ} is restricted to oblivious training.   The extension to more efficient training techniques, such as $\epsy$-greedy or approaches based on Thompson sampling, is another topic for future research. 

\medskip

\textit{Parts of the new material in this survey have been expanded and submitted for publication}.   See 
\cite{laubusmey23} for a fuller development of Zap Zero algorithms,   and 
\cite{mey23b} for theory of Q-learning.

\section{Stochastic Approximation and Reinforcement Learning}
\label{s:SARL}

This section is devoted to three topics:   assumptions surrounding the 
Markov Decision Process (MDP) model,    
a brief summary of results from the theory of stochastic approximation,
followed by 
assumptions surrounding the Q-learning algorithms to be considered.

\subsection{Markov Decision Process}

The first set of assumptions and notation concern the   control system model.  

While the search for an optimal policy may be restricted to static state feedback under the assumptions imposed below,  in reinforcement learning it is  standard practice to introduce randomization  in policies as a way of introducing exploration during training.   We restrict to randomized policies of   the form,
\begin{equation}
U_k = \fee(X_k,\theta_k, I_k)\,,\qquad k\ge 0\,,  \qquad 
\label{e:TrainingPolicy}
\end{equation}
in which $\bfmI=\{I_1,I_\infty ,\dots\}$ is an i.i.d.\ sequence.   Under the assumption that $\state$ and $\ustate$ are finite, we can assume without loss of generality that $\bfmI$ evolves on a finite set.

The input-state dynamics are assumed to be defined by a controlled Markov chain, with controlled transition matrix $P$.   
For any randomized stationary policy, 
\begin{equation}
\Prob\{ X_{k+1} = x' \mid X_k =x\,,\ U_k =u\} = P_u(x,x')\,,\qquad x,x'\in\state\,,\ u\in\ustate\,,\ k\ge 0
\label{e:Pu}
\end{equation}
The dynamic programming equation \eqref{e:DCOE-Q} may be expressed
\begin{equation}
Q^\star(x,u) =  c(x,u) + \disc    \sum_{x'\in\state}   P_u(x,x')\  \uQ^\star ( x')  \,,\qquad x\in\state\,,\ u\in\ustate
\label{e:DCOE-Qb}
\end{equation}

\begin{subequations}

\subsection{What is stochastic approximation?}  
\label{s:SAcrash}

 A fuller answer  may be found in any of the standard monographs, such as \cite{bor20a} (see also \cite{CSRL} for a crash course).

The goal of SA is to solve the root finding problem $\barf(\theta^*) =\Zero$,  where the function is defined in terms of an expectation,  $\barf(\theta) = \Expect[ f(\theta,\Phi)]  $ for $\theta\in\Re^d$ and with $\Phi$ a random vector.     
The general SA algorithm is expressed in two forms:
\begin{align} 
\theta_{n+1} &= \theta_n + \alpha_{n + 1} f(\theta_n \,, \Phi_{n+1} )
\label{e:SAa}
\\
 & = \theta_n +\alpha  [ \barf(\theta_n) + \Delta_{n+1} ]\,,\quad n\ge 0.
\label{e:SAintro}
\end{align}
where   \eqref{e:SAintro} introduces the notation  $ \Delta_{n+1} \eqdef  f(\theta_n, \Phi_{n+1}) - \barf(\theta_n) $.   
It is assumed that the sequence of random vectors $\{ \Phi_n \}$ converges in distribution to $\Phi$.

 \end{subequations}

The   algorithm is motivated by ordinary differential equation (ODE) theory,  and this theory plays a large part in establishing convergence of \eqref{e:SAa} along with convergence rates.    These results are obtained by comparing solutions \eqref{e:SAa} to solutions of the \textit{mean flow},
\begin{equation}
  \ddt \odestate_t = \barf(\odestate_t).
\label{e:ODESA1}
\end{equation}
In particular,  $\theta^\ocp $ is a stationary point of this ODE.

\wham{Averaging}  
A large step-size  $\{ \alpha_{n + 1} \}$  in  \eqref{e:SAa} is desirable for quick transient response, but  this typically leads to high variance.    There is no conflict if the ``noisy'' parameter estimates are averaged.     The  averaging technique of Polyak and Ruppert defines
\begin{equation}
\thetaPR_n = \frac{1}{n}\sum_{k=1}^n \theta_k \,,\qquad n\ge 1.
\label{e:thetaPR}
\end{equation}
 \Cref{t:chedevborkonmey21} illustates the value of this approach.

\wham{Basic SA assumptions} 
The following are imposed in this section, and in some others that follow.

It is assumed that  the step-size sequence   $\{\alpha_n: n \geq 1\}$ is deterministic, satisfies $0 < \alpha_n \leq 1$, and
\begin{equation}
\sum_{n = 1}^\infty \alpha_n = \infty , \quad \sum_{n = 1}^\infty \alpha_n^2 < \infty
\label{e:step-sizeConditions}
\end{equation} 
Much of the theory in this paper is restricted to the special case: $\alpha_n =  g n^\rho$ with $\half <\rho \le  1$ and $g>0$.
We sometimes require two time-scale algorithms in which there is a second step-size sequence  $\{\stepf_n: n \geq 1\}$ that  is relatively large:
\begin{equation}
\lim_{n\to\infty}\frac{\alpha_n}{\stepf_n} = 0
\label{e:gammaalpha}
\end{equation}

\wham{SA1}   The function $\barf$ is globally Lipschitz continuous

\wham{SA2}  $\bfPhi$ is a time-homogeneous Markov chain that evolves on a finite set, with unique invariant pmf $\uppi$.

\wham{SA3}  The mean flow
\eqref{e:ODESA1} is  globally asymptotically stable, with unique equilibrium $\theta^*$.  

The final assumption is needed to obtain useful bounds on the rate of convergence, which requires the existence of a linearization (at least in a neighborhood of $\theta^*$).  Denote
\begin{equation}
 A(\theta) = \partial_\theta \barf\, (\theta)
\label{e:Atheta}
\end{equation}

\wham{SA4} 
The derivative \eqref{e:Atheta}  is a continuous function of $\theta$,    and
 $A^*\eqdef A(\theta^*)$ is a Hurwitz matrix  (its eigenvalues lie in the strict left hand plane).

\smallskip

 Assumptions (SA1)--(SA3) imply convergence of $\{\theta_n\}$ to $ \theta^*$ almost surely from each initial condition, provided one more property is established:  
 \begin{equation}
 \parbox{.6\hsize}{\raggedright
 The  parameter sequence $\{\theta_n : n\ge 0 \}$ is bounded with probability one from each initial condition.
}
\label{e:thetaBdd}
\end{equation}

Verification of (SA3) is typically achieved through a Lyapunov function analysis.   
Lyapunov techniques also provide a means of establishing \eqref{e:thetaBdd}.  One approach is described next.

\wham{ODE@$\infty$}
The so-called Borkar-Meyn theorem  of \cite{bormey00a, bor20a}  
is one approach to establish  \eqref{e:thetaBdd}. 
This result concerns the  time-homogeneous ODE   $\ddt x =  \barfinf(x)$  (the `ODE@$\infty$'') with 
 vector field,
\begin{equation}
\barfinf(\theta)
\eqdef
\lim_{r\to\infty} r^{-1}   \barf(r\theta).
\label{e:barfinfty}
\end{equation}
We always have $\barfinf(0) =0$, which means that the origin is an equilibrium for the
 ODE@$\infty$.     It is also radially homogeneous,   $\barfinf(r\theta) =r \barfinf(\theta) $ for any $\theta\in\Re^d$ and $r>0$.    Based on these properties it is known that local asymptotic stability of the origin implies global exponential asymptotic stability    \cite{bormey00a}.   

The following is an alternative to the  ODE@$\infty$ criterion,  which is equivalent whenever the limit \eqref{e:barfinfty} exits for each $\theta$:

\wham{(v4)}  For a globally Lipschitz continuous and $C^1$ function $V\colon\Re^d\to\Re_+$,
and a constant $\delta_v>0$,
\begin{equation} 
	\frac{d}{dt} V(\odestate_t)  \leq  -  \delta_v    V(\odestate_t)    \,, \qquad   \textit{when $\|\odestate_t\|\ge  \delta_v^{-1} $.}
	\label{e:ddt_bound_LyapfunTmp}
\end{equation}
The use of the designation ``v4'' comes from an anolagous bound appearing in stability theory of Markov chains \cite{MT}.

It is shown in \cite{bormey00a} that \eqref{e:thetaBdd} holds  provided the ODE@$\infty$ is locally asymptotically stable,  and $\{\Delta_n\}$ appearing in \eqref{e:SAintro} is a martingale difference sequence.
This statistical assumption does \textit{not} hold in many applications of reinforcement learning.   Relaxations of the assumptions of   \cite{bormey00a} are given in   \cite{bha11,rambha17}, but the story is far from complete.

A generalization appeared recently in \cite{chedevborkonmey21} that requires minimal assumptions on the Markov chain (there is no need for a finite state space).    Conclusions obtained under the assumptions imposed here are summarized in the following.

 \begin{theorem}
\label[theorem]{t:chedevborkonmey21}
Suppose that (SA1) and (SA2) hold for the SA recursion \eqref{e:SAa}, and in addition that the origin is locally asymptotically stable for the  ODE@$\infty$,  or that (v4) holds.    Then,

\whamrm{(i)} 
The bound  \eqref{e:thetaBdd} holds in a strong sense:
there is a fixed constant $B_\tTheta$ such that for each initial condition $(\theta_0,\Phi_0)$,
\begin{equation}
\limsup_{n\to\infty} \| \theta_n \|   \le B_\tTheta   \quad a.s..
\label{e:ODEbdd}
\end{equation}

\whamrm{(ii)}   If in addition (SA3) holds then $\displaystyle \lim_{n\to\infty} \theta_n = \theta^*$ almost surely from each initial condition.

\begin{subequations}

\whamrm{(iii)} 
Suppose   that  (SA1)--(SA4) hold, and that   $\alpha_n =  g n^\rho$,  $n\ge 1$,  with $\half <\rho < 1$ and $g>0$. We then have  convergence in mean square, and the following limits exist and are finite:
\begin{align}
\lim_{n\to\infty}  \frac{1}{\alpha_n}  \Expect[\tiltheta_n \tiltheta_n^\transpose ]  & =  \SigmaTheta
\label{e:covThm}
\\
\lim_{n\to\infty}   n\Expect[\tilthetaPR_n {\tilthetaPR_n}^\transpose ]  &=  \SigmaPR
\label{e:PRcovThm}
\end{align}%
 \end{subequations}%
\qed
\end{theorem}

 The covariance matrix $\SigmaPR$ is   minimal in a matricial sense, made precise in \cite{rup88,pol90}.   It has the explicit form $ \SigmaPR =   G  \Sigma_{\Delta}^* G^\transpose  $ in which $G = - ( A^\ocp )^{-1}$,  the stochastic Newton-Raphson gain of Ruppert \cite{rup85},   and   $\Sigma_{\Delta}^* $ is
 the asymptotic covariance 
  \begin{equation}
\begin{aligned}
\Sigma_{\Delta}^* = \sum_{k=-\infty}^\infty \Expect_\pi [ \Delta_k^*  \{\Delta_k^*\} ^\transpose]
\end{aligned}
\label{e:SigmaMD}
\end{equation}
where   $\{ \Delta_k^* \eqdef f(\theta^*, \Phi_{k})  : k \in\intgr\}$,  with 
 $\bfPhi$   a stationary version of the Markov chain
 on the two-sided time interval.
 
\wham{A criterion for stationary points}  

The existence of a suitable Lyapunov function implies the existence of a stationary point.


\begin{proposition}[Lyapunov Criterion for Existence of a Stationary Point]
\label[proposition]{t:ODEstableImpliesPBE}
For an ODE \eqref{e:ODESA1} with globally Lipschitz continuous vector field,  suppose there is a function $V\colon\Re^d\to\Re_+$ with locally Lipschitz continuous gradient, satisfying  for some $\bdd{t:ODEstableImpliesPBE}$,
 \[
 \nabla V(\theta)^\transpose \barf(\theta)  \le -1\,,\quad  \textit{whenever $\|\theta\|\ge \bdd{t:ODEstableImpliesPBE}$.}
 \]
 Suppose moreover that $V$ is convex and coercive.   Then there exists a solution to  $  \barf(\theta^*) = \Zero$.
\end{proposition}

\begin{proof}
 Let $L_\delta(\theta) = \theta + \delta \barf(\theta)$ for $\theta\in\Re^d$,  with   $\delta>0$ to be chosen.
  For $\delta>0$ sufficiently small we construct a convex and compact set $S_\delta$ for which $L_\delta(\theta) \in S_\delta$ for each $\theta \in S_\delta$.  It follows from Brouwer's fixed-point theorem that there is a 
solution to  $L_\delta(\theta^*) = \theta^*  $.   This is equivalent to the desired conclusion   $  \barf(\theta^*) = \Zero$. 

Denote $b_\delta = \sup\{ V( L_\delta(\theta) )  :   \|\theta\| \le \bdd{t:ODEstableImpliesPBE} \}$,  and $S_\delta = \{ \theta :  V( \theta )\le b_\delta \}$;   a   convex and compact set subject to the assumptions on $V$.

We next show that $S_\delta $ is invariant under $L_\delta$ if $\delta$ is small.  We consider two cases,  based on whether or not $\theta$ lies in the set  $S = \{ \theta :  \|\theta\| \le \bdd{t:ODEstableImpliesPBE} \}$

\wham{1.}
If $\theta \in S_\delta\cap S$,   then $ L_\delta(\theta) \in 
S_\delta $ by construction of $S_\delta$.   

\wham{2.}  
If  $\theta \in  S_\delta\setminus S $ then we apply convexity combined with the drift condition:    denoting $\theta^+ = L_\delta(\theta) $,  
\[
V(\theta) \ge V ( \theta^+ ) +\nabla V ( \theta^+ )^\transpose (\theta -\theta^+) 	
			= V ( \theta^+ )  - \delta \nabla V ( \theta^+ )^\transpose \barf(\theta)    
\]
Since the gradient is locally Lipschitz continuous and $\barf$ is globally Lipschitz continuous, there is $b_v$ satisfying
 \[
V(\theta) \ge V ( \theta^+ )  - \delta \nabla V ( \theta)^\transpose \barf(\theta)    -b_v\delta^2\,, \qquad \theta \in  S_\delta\setminus S
\]
The value of $b_v$ can be chosen independent of $\delta\in (0,1]$.

Under the assumed drift condition this gives $ V ( \theta^+ )  \le V(\theta)   - \delta + b_v\delta^2 $.  Choosing $\delta=1/b_v$ gives 
$ V ( \theta^+ )  \le V(\theta)  \le b_\delta$, in which the second inequality holds because $\theta \in  S_\delta\setminus S $.  Hence $ L_\delta(\theta)  = \theta^+\in S_\delta$ as desired. 
\end{proof}

\subsection{Compatible assumptions for Q-learning} 

The basic Q-learning algorithm \eqref{e:Qlambda} is an instance of stochastic approximation, for which we can apply general theory subject to assumptions on the input used for training (recall \eqref{e:TrainingPolicy}).

Two settings are considered:

\wham{Oblivious training}  
This means that \eqref{e:TrainingPolicy} simplifies to 
\begin{equation}
U_k = \fee(X_k , I_k)\,,\qquad k\ge 0\,, 
\label{e:TrainingPolicyOb}
\end{equation}
in which it is always assumed that  $\{I_k\}$ is   i.i.d..

It follows that the pair process  $ (X_k , U_k) : k\ge 0\}$ is a time homogeneous Markov chain.     It is assumed to be uni-chain (i.e., the invariant pmf $\uppi$ is unique). 
 In the expression $ f_{n+1}(\theta_n) = f(\theta_n \,, \Phi_{n+1} )$ we take  
$\{\Phi_k = (X_k;X_{k+1}; U_k) : k\ge 0\}$, which is also a time homogeneous Markov chain, for which its invariant pmf is also unique and easily expressed in terms of $\uppi$ and the controlled transition matrix.  

If the function class is linear  $\{Q^\theta = \theta^\transpose \psi : \theta\in\Re^d\}$,   then the autocorrelation matrix is assumed full rank  
\begin{equation}
R_0 = \Expect_{\uppi}[ \psi(X_n,U_n)\psi(X_n,U_n)^\transpose]
\label{e:R}
\end{equation}
where the expectation is taken in steady-state

\wham{Optimistic training} 

In this non-oblivious approach the input sequence depends on  the parameter sequence,  and is designed to approximate the Q-greedy policy \eqref{e:feetheta}.   There are only a finite number of deterministic stationary policies, so  $\fee^\theta$ is necessarily discontinuous in $\theta$.   
The region on which continuity holds is denoted
\begin{equation}
\feeUnique =  
\left\{ 
\parbox{.72\hsize}{\raggedright  
$  \theta\in\Re^d :   $  there is $\epsy>0$ s.t.\   $\fee^\theta(x) =  \fee^{\theta'}(x) $ 
for all $x$ when $\| \theta - \theta'\|\le \epsy $ 
}
\right\}
\label{e:Sunique}
\end{equation}

The training policy is taken of the form,
\begin{equation}
U_k =  (1 - B_k) \clU_k  +  B_k \EXP_k
\label{e:epsyGreedy}
\end{equation}
in which $\{ B_k \} $ is an i.i.d.\ Bernoulli sequence with $\Prob\{B_k =1\} =\epsy$,   and $\{ \EXP_k \} $ is an i.i.d.\  sequence taking values in $\ustate$ and independent of $\{ B_k \} $.     
The $\ustate$-valued random variable   $ \clU_k $ depends on the parameter $\theta_k$, and is independent of  $( B_k ;\EXP_k)$ for each $k$.      
  
The sequences $\{U_k,\clU_k : k\ge 0\}$  are defined by  randomized stationary policies $\{\feex^{\theta}  \,, \, \feex^{\theta}_0     :  \theta\in\Re^d\}$.  
Both $\feex^{\theta} (\varble \mid x)$ and $\feex^{\theta}_0 (\varble \mid x)$ are  pmfs on $\ustate$ for each $x$ and $\theta$.   
Based on the assumptions imposed after \eqref{e:epsyGreedy}, we have  
\begin{equation}
\begin{aligned}
\Prob\{ U_k = u\mid  \clF_k^-;   X_k = x   \}    
&=
\feex^{\theta_k} (u \mid x)  
\\
&=(1-\epsy) \feex^{\theta_k}_0 (u \mid x)    +  \epsy \upnu_\tEXP(u) 
\end{aligned} 
 \label{e:OptimisticGeneric}
\end{equation}
with      $\upnu_\tEXP$ the common  pmf for $\{ \EXP_k\}$,
and $\clF_k^- = \sigma\{  X_i, U_i : i<k ;   B_i, \EXP_i : i\le k    \} $  (a partial history of observations up to iteration $k$).

The special cases are described in the following.

 \wham{1. \boldmath  $\epsy$-greedy.}
 Recalling the definition of $\fee^\theta$ in \eqref{e:feetheta},  the  $\epsy$-greedy  policy is defined by the choice $\clU_k  =  \fee^{\theta_k} (X_k)$, so that 
\begin{equation}
\feex^{\theta}_0 (u \mid x)  =   \ind\{u =  \fee^\theta(x) \}  
\label{e:epsyGreedyTrue}
\end{equation}

The mean flow has many attractive properties  (see \Cref{t:QstableODEGreedy} in the Appendix).
However,  because $\{ \fee^\theta : \theta\in\Re^d\}$ is a piecewise constant function of $\theta$,  it follows that  
the vector field $\barf$ is not continuous in $\theta$ as required in \Cref{t:chedevborkonmey21}.  

\wham{2.  Gibbs approximation}
Fix a large constant $\kappa>0$ and define  
\begin{equation}
\feex^{\theta}_0(u \mid x) 
 =  \frac{1}{ \clZ^\theta_\kappa  (x)}   \exp\bigl(  -\kappa  Q^\theta(x,u)     \bigr)  
\label{e:epsyGreedySoftGibb}
\end{equation}
in which $\clZ^\theta_\kappa  (x)$ is normalization.     
This is indeed an approximation of \eqref{e:epsyGreedyTrue}: 
for $\theta^0\in \feeUnique$, 
\begin{equation}
\lim_{r\to\infty}   \frac{1}{ \clZ^{r\theta}_\kappa  (x)}   \exp\bigl( - \kappa Q^{r\theta}(x,u)     \bigr)    =    
 \ind\{ u = \fee^{\theta^0}(x) \}
\label{e:GibbsDegeneracy}
\end{equation}

The limit \eqref{e:GibbsDegeneracy}  has two important implications.  First is that the vector field  $\barfinf$  for the  ODE@$\infty$ is unchanged whether we consider 
\eqref{e:epsyGreedyTrue} or its smooth approximation \eqref{e:epsyGreedySoftGibb}.    Second is that discontinuity of $\barfinf$  implies that $\barf$ is not globally Lipschitz continuous, which violates an assumption of \Cref{t:chedevborkonmey21}.

 \wham{3. Tamed Gibbs approximation}
This is a modification of \eqref{e:epsyGreedySoftGibb} in which $\kappa$ depends on $\theta$:
\begin{equation}
\feex^{\theta}_0 (u \mid x) 
 =    \frac{1}{ \clZ^\theta_\kappa  (x)}   \exp\bigl( - \kappa_\theta  Q^\theta(x,u)     \bigr)  
\label{e:epsyGreedySoftGibbTamed}
\end{equation}
For analysis the following structure is helpful:
choose a large constant $\kappa_0>0$,   and  assume  that
 \begin{equation}
\kappa_\theta \    \begin{cases}    =  \frac{1}{ \|\theta\|}  \kappa_0&   \|\theta\|\ge 1
\\				 \ge \half \kappa_0   &\textit{else}
	\end{cases}  
\label{e:normGibbsBdds}
\end{equation}
This will be called the  \textit{$(\epsy,\kappa_0)$-tamed Gibbs policy} when it is necessary to make the policy parameters explicit.

  The equality in \eqref{e:normGibbsBdds} ensures the following identity holds  all $x,u$:
 \begin{equation}
\text{   $\feex^{r\theta} (u \mid x) =  \feex^\theta(u \mid x) $ for all $r\ge 1$   and $\|\theta\| \ge 1$.}
\label{e:LargeParPolicy}
\end{equation}

\medskip

The Q-learning algorithm \eqref{e:Qintro} can be cast as stochastic approximation when the input is defined using any of the training policies described above,   in which  we take $\Phi_{n+1} =  (X_n, X_{n+1} , U_n)$ since  these three variables appear in  \eqref{e:Qintro}.

It is assumed in \Cref{t:chedevborkonmey21} that $\bfPhi$ is \textit{exogenous}---its transition matrix does not depend on the parameter sequence.     
Fortunately, there is now well developed theory that allows for parameter-dependent dynamics for $\bfPhi$ in the SA recursion \eqref{e:SAa}---see the recent paper  
\cite{faibor23} for history and recent results.   In particular,  theory of  convergence  and asymptotic statistics is now mature.

The question is then, \textit{how can we apply SA theory to make statements about convergence and  convergence rates?}

\section{Trouble with Tabular}
\label{s:Q}

\subsection{Linear function approximation}

In this section we  restrict to a linearly parameterized family  $\{Q^\theta \eqdef \theta^\transpose  \psi\colon \theta \in \Re^d \}$,  where $\psi\colon \state\times\ustate \to \Re^d$ is a vector of basis functions. 
To avoid long equations we often use the shorthand notation,
\begin{equation}
\csub{n} = c(X_n,U_n)\,,\quad 
\psisub{n} = \psi(X_n,U_n)\, ,
\quad f_{n+1}(\theta_n) = f(\theta_n \,, \Phi_{n+1} )
\label{e:subNotation}
\end{equation}
   In the recursion \eqref{e:Qlambda} we then have $ \elig_n  = \nabla_\theta  Q^\theta(X_{n}, U_{n})  = \psisub{n}$.

When considering optimistic policies we encounter an additional complication in the description of the vector field for the mean flow.    If the input is of the form  \eqref{e:OptimisticGeneric}, then   for each $\theta$ we consider the resulting transition matrix for the joint process $ \{ (X_k,U_k)  : k\ge 0\}$ defined by
\begin{equation}
T_\theta(z,z')  \eqdef    P_u(x,x') \feex^{\theta} (u' \mid x')  \,, \qquad  z=(x,u)\, ,\  z'=(x',u')  \in \state\times\ustate\, .
\label{e:Ttheta}
\end{equation}
where $P$ is the controlled transition matrix.   It is assumed that each admits a unique invariant pmf $\uppi_\theta$.

\begin{subequations}

Q-learning in the form 
\eqref{e:Qlambda} is an instance of stochastic approximation,  with mean flow
 \begin{align}
  \barf(\theta)  &=   \Expect_{\uppi_\theta} [  \psisub{n}    \BelErr  (X_n,U_n;\theta)]\,,  
\label{e:ProjBE}
   \\
            &  				  \BelErr (x,u;\theta)  =   c(x,u)    - Q^{\theta}(x,u) + \disc   \sum_{x'} P_u(x,x')  \uQ^{\theta} (x')
  \label{e:BE}
\end{align} 
An alternative formula is valuable for analysis,
\begin{equation}
\begin{aligned}
&  \barf(\theta) = A(\theta) \theta - b(\theta) 
   \\
   \textit{with}\quad
A(\theta )= - \Expect_{\uppi_\theta} & \bigl[  \psisub{n} \{  \psisub{n} - \disc \psi(X_{n+1},\fee^\theta(X_{n+1})  \} ^\transpose \bigr] \,,
  \quad b(\theta) = - \Expect_{\uppi_\theta} [  \psisub{n} \csub{n} ]
  \end{aligned}  
\label{e:barfQgen}
\end{equation}
 The  projected Bellman equation \eqref{e:PBE}  is precisely the root finding problem,  $\barf(\theta^*) = \Zero$.

 \end{subequations}

Any choice of oblivious policy fits the standard SA theory, with $\barf$ globally Lipschitz continuous.  
The tamed Gibbs  approximation is the only choice among the optimistic training rules for which $\barf$ satisfies the smoothness conditions required in \Cref{t:chedevborkonmey21}.
   \notes{add ref to vivek on  discontinuous RHS.   See his papers.  }

\smallskip

\textit{In the remainder of this section we  restrict to   oblivious training.    
}

\subsection{Tabular Q-learning, the good and the bad}
\label{s:tabQ}

 In the tabular setting we have $d=|\state|\times|\ustate|$:  given an ordering of the state-action pairs $\{ (x^i,u^i) : 1\le i\le d\}$ we take for each $i$,
\begin{equation}
\psi_i(x,u) = \ind\{ (x,u) =  (x^i , u^i) \} \,,\qquad x\in\state\,,\ u\in\ustate
\label{e:tabBasis}
\end{equation}
In view of  \eqref{e:Qlambda} we find that only one entry of the parameter is updated at each iteration.   It is typical to use a diagonal matrix gain,
\begin{equation}
\theta_{n+1} = \theta_n + \alpha_{n+1}  G_n  \Tdiff_{n+1} \elig_n 
\label{e:matrixGainQ}
\end{equation}
in which $G_n^{-1} (i,i) $ indicates the number of times the pair $(x^i , u^i) $ is visited up to time $n$   (set to unity when this is zero).

Observe that by definition $Q^{\theta_n}(x^i,u^i) =\theta_n(i)$.   Adopting the notation $q_t$ instead of $\odestate_t$ for the ODE state in the mean flow \eqref{e:ODESA1}  associated with the matrix gain recursion \eqref{e:matrixGainQ}, we have
\begin{equation}
\begin{aligned}
\ddt q_t  &=    A(q_t)  q_t  -    b
\end{aligned} 
\label{e:ODEq}
\end{equation}
 with $b$ the $d$-dimensional vector with entries $b_i = - c(x^i, u^i)$.
The matrix-valued function $A$ is piecewise constant:
\begin{equation}
A(q)  = -     [I - \disc    T(q)  ] \,,\qquad
T^{i,j} (q)  =  P_{u^i}(x^i, x^j)  \ind\{ u^j = \fee^q(x^j)  \}  
\label{e:QsimpleWatkins}
\end{equation}

 \wham{The good news:}  
 The statistical properties of the algorithm are attractive because   $\{  \Delta_{n+1}  \}$ appearing in \eqref{e:SAintro} is a  martingale difference sequence in 
 the tabular setting.
 
  The best news is stability:
 The induced operator norm of  $T(q)$ in $\ell_\infty$ is less than one, meaning $\max_i | \sum_j T_{i,j} (q) v_j  | \le \| v\|_\infty \eqdef \max_i |    v_i|$ for any vector $v$ and any $q$.   It follows that the $\ell_\infty $ norm serves as a Lyapunov function:
  Letting    $\tilq_t = q_t-\theta^*$ and $V(q) = \| \tilq_t \|_\infty $, 
 \[
 \ddt  V(q_t)  \le -(1-\disc) V(q_t) 
 \]
 This is how convergence is established for tabular Q-learning.
 
 \wham{The bad:}   
 The matrix $ I - \disc    Z(q)$ has an eigenvalue at $(1-\disc)$ for all $q$,   which is a reason for slow convergence when the discount factor is close to unity.

\begin{figure}
\centering
\includegraphics[width=.9\hsize]{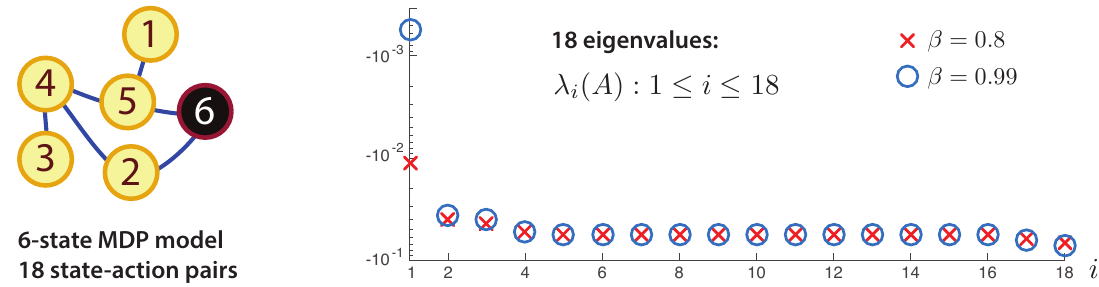}  

\caption{Graph for MDP and eigenvalues of $A^*$ }	
\label{f:6StateGraph+eigs}
 \end{figure}

It is now known that  the asymptotic covariance $\SigmaTheta$ appearing in 
\eqref{e:covThm} is not finite if  $\disc>1/2$  \cite{devmey17a,devmey17b}  (see also the sample complexity analysis that followed in \cite{wai19a}).
A running example in this prior work and \cite{CSRL,devbusmey21,dev19} is the   stochastic-shortest-path problem
whose state transition diagram is shown on the left hand side of \Cref{f:6StateGraph+eigs}.   The  state space $\state = \{1,\ldots,6\}$ coincides with the  six nodes of the un-directed graph, and the action space is $\ustate =\{ e_{x,x'} \}$,
$x,x' \in \state$, indicating decisions on moves.      Details on the description of disturbances can be found in    \cite{CSRL,devbusmey21}.

\Cref{f:AsymCov_Watg1}   shows results without the matrix gain.    With $\disc=0.8$ the output of the standard Q-learning algorithm is worthless after one million samples.  The matrix gain does offer some benefit---see plots in the next section---but convergence remains very slow for $\disc>\half$ using $\alpha_n = 1/n$.

 \begin{figure}[h]
\centering 
  \includegraphics[width= .75\hsize]{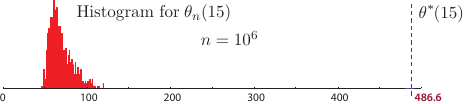}
\caption{Histogram of $10^3$ estimates of $\theta_n(15)$, with $n=10^6$ for the Watkins algorithm applied to the 6-state example with discount factor $\disc = 0.8$}
\label{f:AsymCov_Watg1}
\end{figure}

\wham{ODE@$\infty$}

\Cref{t:chedevborkonmey21} may be applied to  Q-learning \eqref{e:Qlambda} in this tabular setting,  and the theorem easily extends to the case of the matrix gain algorithm \eqref{e:matrixGainQ}.

  It is clear that (SA1) holds, and (SA2) holds for oblivious training as assumed here.   As already remarked,  it is not difficult to establish stability of  \eqref{e:ODEq} to establish (SA3).

The ODE@$\infty$ associated with 
\eqref{e:ODEq} is a minor modification:
\[
\ddt x_t   =   A(x_t)   x_t   
\]
We have $ \ddt \| x_t  \|_\infty         \le -(1-\disc)  \| x_t  \|_\infty   $,  which implies that the ODE@$\infty$ is stable as required in \Cref{t:chedevborkonmey21}.

\subsection{Change your goals}

 Recall that the covariance $\SigmaTheta$ defined in  \eqref{e:covThm} is  not finite   for Q-learning in the form \eqref{e:matrixGainQ}  with step-size $\alpha_n = g/n$  using $g < \half (1-\disc)^{-1}$, which explains the poor performance illustrated in   \Cref{f:AsymCov_Watg1}.
 
 A reader with experience in SA would counter that this is a poor choice of step-size.   Use instead  $\alpha_n = 1/n^\rho$,  with $\rho\in (\half, 1)$,  and then average using
\eqref{e:thetaPR} to obtain $\{\thetaPR_n \}$.  It is found that averaging fails for this example for large discount factors, even though it is known that these estimates achieve the optimal asymptotic covariance  \cite{CSRL,devbusmey21,dev19}.

The observed numerical instability  is a consequence of the eigenvalue at $-(1-\disc)$ for $A^*$.
This can be moved through a change in
objective.   For example, construct an algorithm that estimates the \textit{relative Q-function},
 \[
 H^\star
(x,u) = Q^\star
(x,u)  -  \langle \upnu,  Q^\star
 \rangle 
 \]
 where $\upnu$ is a fixed pmf on $\state\times \ustate$.   Subtracting a constant doesn't change the minimizer over $u$,  and has enormous benefits. 

  The function $H^\star
$ satisfies a DP equation similar to   \eqref{e:DCOE-Qb}, which motivates \textit{relative Q-learning}.   It is shown in  \cite{devmey22} that the eigenvalues of $A^*$ remain bounded away from the imaginary axis uniformly for all $0\le \disc \le 1$, resulting in much faster convergence.
See \cite{CSRL} for generalizations.  

 \Cref{f:rel} is adapted from  \cite{devmey22}  for the six-state example.   The plots show  the span  semi-norm error $\|Q^{\theta} - Q^\star
\|_S$ for three algorithms, and two very large discount factors:
 $\disc = 0.999$ and $\disc = 0.9999$.   The plots illustrate two important points:

 \wham{1.}    Q-learning with a smaller gain converges quickly to $Q^*$ when measured in the span norm
 \[
 \|Q^{\theta} - Q^\star
\|_S \eqdef \min_a  \sup_{x,u}  |  Q^{\theta}(x,u) - Q^\star
(x,u) |   \,,\qquad \theta\in\Re^d
 \]
  \wham{2.}  Convergence of relative Q-learning is very fast in this example, even with discount factor close to unity.   
  
It appears that the span norm difference between estimates obtained using Q-learning    and relative Q-learning is very small.   This observation may be anticipated by comparing the respective mean flows     \cite{devmey22}.

\begin{figure}[h]
\centering
 \includegraphics[width=0.95\textwidth]{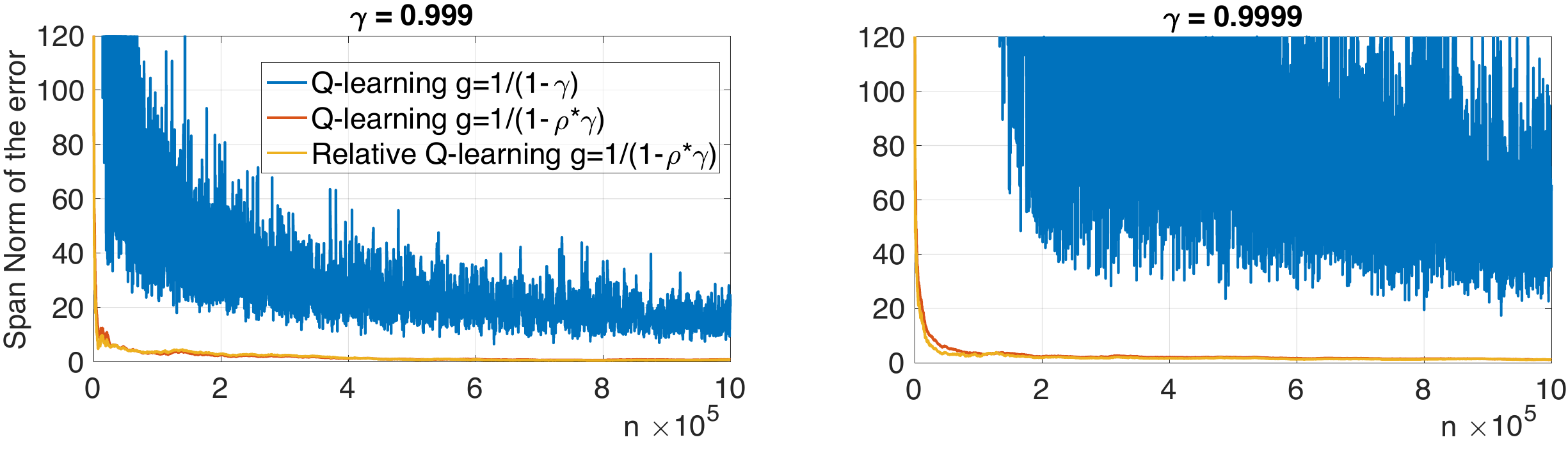} 
 
 \caption{Plots of the span semi-norm of the error $Q^{\theta_n} - Q^\star
$ for two versions of  Q-learning distinguished by the step-size, and relative Q-learning.  }
 \label{f:rel}
 \end{figure}

\section{Zap}
\label{s:Z}

Here the tabular setting is abandoned, and we do not even require linear function approximation. 
We maintain the assumption that the input for training is oblivious.
 
If our goal is to ensure that $\barf(\odestate_t)\to 0$ as $t\to\infty$ then we should \textit{design} dynamics to ensure this.    One approach, the focus of Devraj's dissertation  \cite{dev19} and a focus of the   monograph \cite{CSRL}, is the 
\textit{Newton-Raphson Flow}:
\begin{equation}
 \ddt \barf(\odestate_t) = - \barf(\odestate_t)  
 \quad \Longrightarrow  \quad
 		\barf(\odestate_t)  = e^{-t} \barf(\odestate_0)
 \label{e:NRflow}
\end{equation}
From the chain rule this results in the mean flow dynamics,
\begin{equation}
\ddt \odestate_t =- G_t  \barf(\odestate_t)  \,,\qquad G_t = A(\odestate_t)^{-1}  
\label{e:zapODE}
\end{equation}
where $A(\theta) $ is defined in \eqref{e:Atheta}.

\begin{figure}[h]
\centering
\includegraphics[width=0.7\hsize]{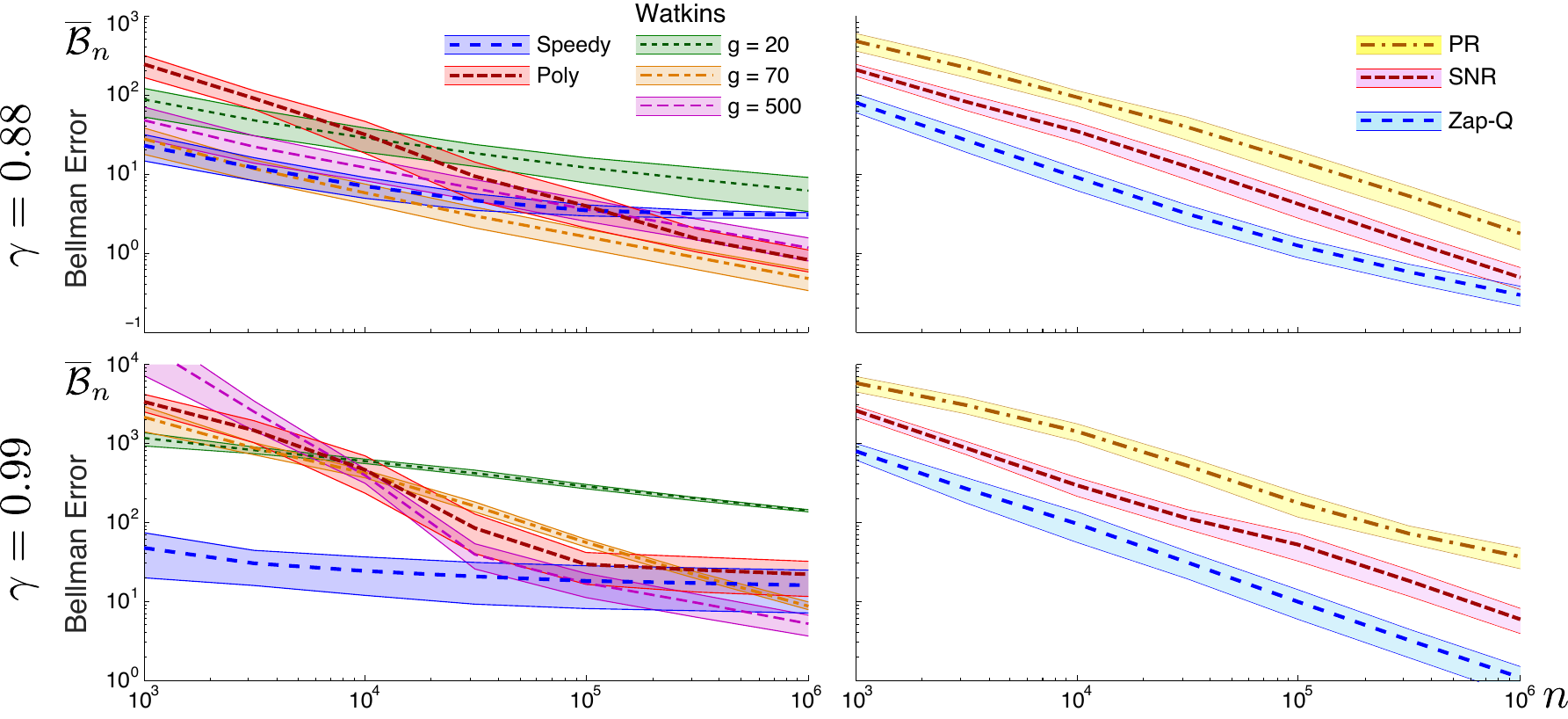}
\caption{Simulation-based $2 \sigma$ confidence intervals for the six Q-learning algorithms.  }
\label{6StateConfIntBEPlot1+2}
\end{figure}

\begin{subequations}

\wham{Zap stochastic approximation.} This is a  two time-scale  algorithm introduced in  \cite{devmey17a,devmey17b}. 
For initialization $\theta_0\,\in\Re^d  $, and $\haA_0 \in \Re^{d \times d}$, obtain the sequence of estimates $\{\theta_n: n \geq 0\}$ recursively:  
\begin{align} 
\theta_{n+1} &= \theta_n -  \alpha_{n+1} \haA_{n+1}^{-1} f(\theta_n \,, \Phi_{n+1})    
\label{e:ZAPnonlinearSAa} 
\\
\haA_{n+1} &= \haA_n + \stepf_{n+1}   [  A_{n+1} - \haA_n   ]\,,
		\qquad A_{n+1} \eqdef \partial_\theta f_{n+1}(\theta_n)  \,.
\label{e:ZAPnonlinearSAb} 
\end{align}
 The   two gain sequences  $\{\alpha_n\}$ and $\{\stepf_n\}$   satisfy  \eqref{e:gammaalpha}.
 
 \label{e:ZAPnonlinearSA} 

 \end{subequations}

The original motivation was  to optimize the rate of convergence,  which it does under mild assumptions using $\alpha_n = 1/n$:
\[
\SigmaTheta \eqdef
\lim_{n\to\infty}  n\Expect[\tiltheta_n {\tiltheta_n}^\transpose ] =  \SigmaPR
\] 
This choice of gain $g$ is critical with the choice $\alpha_n = g/n$:  
\wham{1.}    
$\SigmaTheta$ is finite when $g>\half$, but optimal only if $g = 1$. 
 
\wham{2.}    
 $\trace(\SigmaTheta)=  \infty$      if  $0<g<\half$ and   $\Sigma_{\Delta}^* $ is full rank.

  It was discovered in later research that the positive results hold even for \textit{nonlinear} function approximation
\cite{chedevbusmey20b}.
   Hence the greatest value of the matrix gain is the creation of a universally stable algorithm.

\begin{figure}
\centering

\includegraphics[width=0.7\hsize]{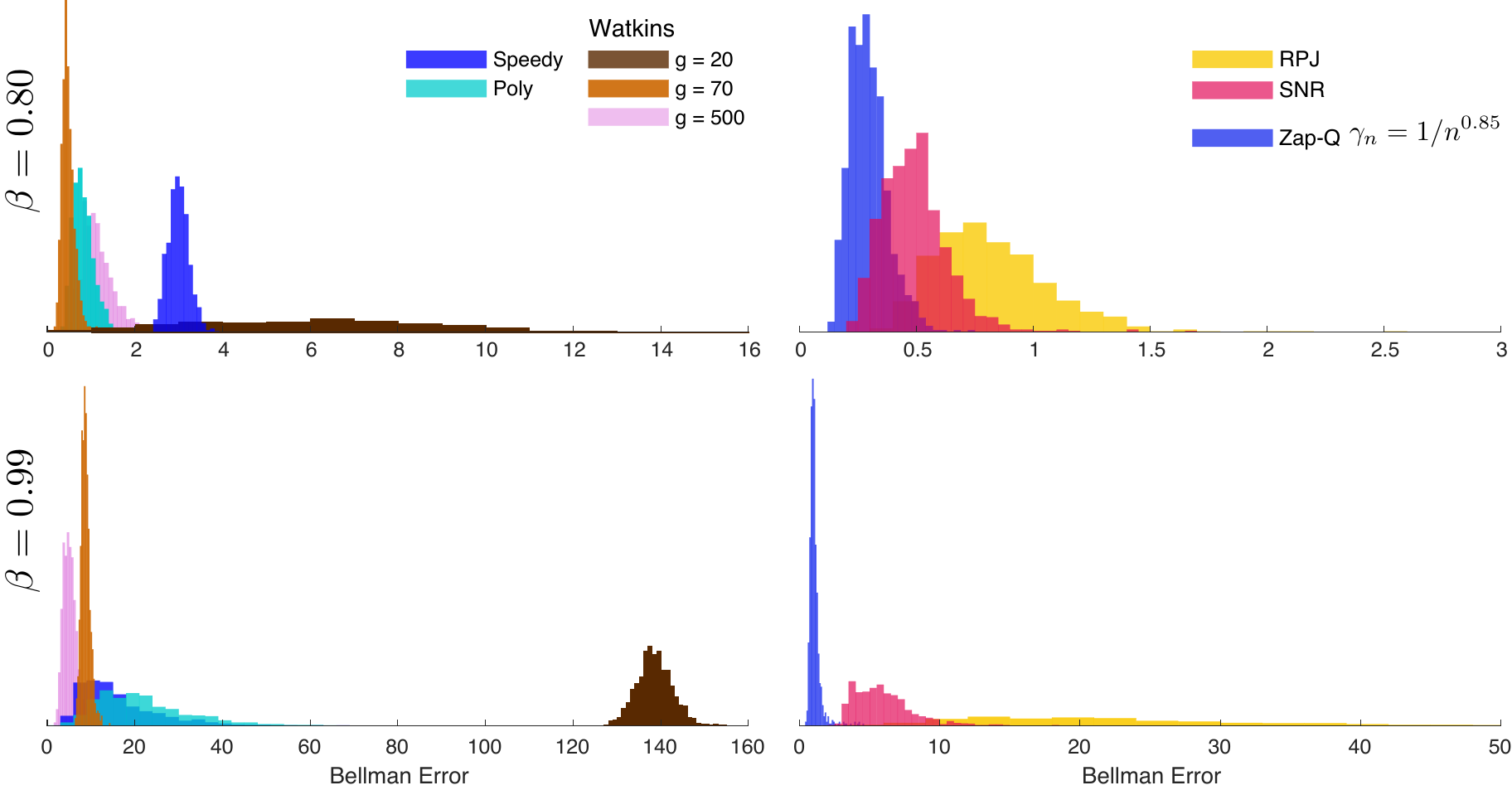}
\caption{Histograms of the maximal Bellman error at iteration $n=10^6$.}
\label{f:6StateHistMaxBEPlot08+99}
\end{figure}

\begin{subequations}

In the applications to Q-learning considered here, the accuracy of a parameter estimate $\theta$ may be measured in terms of the
 Bellman error and its maximum
\begin{align}
\BE^{\theta}(x,u)  &= \theta(x,u) -  r(x,u) - \beta \sum_{x'\in\state}  P_u(x,x')  \max_{u'} \theta(x',u')
\label{e:BError}
\\
\maxBE^{\theta}  &= \max_{x,u}  | \BE^{\theta}(x,u) | 
\label{e:BErrorMax}
\end{align}
Given that the CLT holds for $\sqrt{n} \tiltheta_n$,   the    sequence $\{\sqrt{n} \,\maxBE^{\theta_n}\}$ also converges in distribution as $n\to\infty$.

\end{subequations}

\Cref{6StateConfIntBEPlot1+2}
is taken from
\cite{devbusmey21}
which contains full details on the experiments.  
Plots show the empirical mean and  $2 \sigma$ confidence intervals   for $\disc = 0.8$ in row 1,   and   $\beta = 0.99$ in row 2.    The algorithms considered in the second column are explained here:

\wham{PR}  Estimates obtained from Q-learning and averaging \eqref{e:matrixGainQ}, with step-size 
   $\alpha_n = 1/n^{0.85}$.
\wham{SNR}   The single time scale variant of  Zap Q-learning using $\alpha_n =\stepf_n = 1/n$  (for which no theory is available).
\wham{Zap}   Zap Q-learning using $\alpha_n = 1/n$ and $\stepf_n  = 1/n^{0.85}$.
 
\Fig{f:6StateHistMaxBEPlot08+99}  shows   histograms of  $\{\maxBE_n^i,\, 1 \leq i \leq N\}$, $n = 10^6,$ for all the six algorithms;   this corresponds to the data shown in  \Fig{6StateConfIntBEPlot1+2}   at $n = 10^6$. 
Once again, full details may be found in \cite{devbusmey21}.

\section{Zap Zero}
\label{s:ZZ}

If we denote $w_t = G_t  \barf(\odestate_t) $ in the notation of \eqref{e:zapODE}, then equivalently  $A(\odestate_t) w_t -  \barf(\odestate_t) =0 $ for all $t$.
The Zap Zero algorithm of \cite{CSRL} is designed to achieve this constraint without matrix inversion.   
The ODE method for design suggests the $2d$-dimensional  ODE
\begin{equation}
\begin{aligned}
\ddt \odestate_t &  =  - w_t
\\
\ddt w_t  & =  \stepf_t  \bigl[ A(\odestate_t) w_t -  \barf(\odestate_t)   \bigr]
\end{aligned} 
\label{e:zapzeroODE1}
\end{equation}
in which the time varying gain is introduced in anticipation of a two time-scale SA translation.    
The time inhomogeneous ODE \eqref{e:zapzeroODE1} is stable provided $\stepf_t\uparrow\infty$ as $t\to\infty$, and in addition
(SA4) holds with  $A(\theta) $ Hurwitz for each $\theta$.

\wham{A universally stable algorithm}
A third state variable is introduced for reasons to be explained when we consider the SA translation.    Fix $M>0$, an arbitrary positive definite matrix, and consider the ODE
\begin{equation}
\begin{aligned}
\ddt \odestate_t &  = - [\odestate_t +   z_t    ]
\\
\ddt w_t  &=   - \stepf_t      \bigl[     A(\odestate_t) \{  \odestate_t  +  z_t \}  -  \barf(\odestate_t)   \bigr]
\\
\ddt z_t  &=   - \stepf_t        \bigl[      z_t  -  L_t w_t  \bigr]   &&   L_t = M  A(\odestate_t) ^\transpose
\end{aligned} 
\label{e:zapzeroODEnew}
\end{equation}

The choice $\ddt z_t      - \stepf_t        \bigl[      z_t  -  M  A(\odestate_t) ^\transpose w_t  \bigr]  $ is just one option.  Alternatives are described in \cite{laubusmey23}.

Assuming once more that  $\stepf_t\uparrow\infty$ as $t\to\infty$,  singular perturbation theory (e.g.\ \cite{kokorekha99})  provides methodology for verification of stability of \eqref{e:zapzeroODEnew},  proceeding in two steps:

\wham{1.} Consider the pair of ODEs with the slow variable frozen:  
\[
\begin{aligned} 
\ddt w^\theta_t  &= -A(\theta) \{  \theta + z^\theta_t \} +  \barf(\theta)   
\\
\ddt z^\theta_t  &=                -   z^\theta_t  +  L(\theta) w^\theta_t      &&   L(\theta) = M  A(\theta) ^\transpose
\end{aligned} 
\]
The gain $\stepf_t$ has been removed via a time transformation.   For stability analysis it is more convenient to write,
\begin{equation}
\begin{aligned} 
\ddt 
\begin{bmatrix}  
w^\theta_t 
\\
z^\theta_t 
\end{bmatrix}
=
\begin{bmatrix}  
0  &  -A(\theta)
\\
M  A(\theta) ^\transpose  & -I
\end{bmatrix}
\begin{bmatrix}  
w^\theta_t 
\\
z^\theta_t 
\end{bmatrix}
+
\begin{bmatrix}  
  \barf(\theta)    -A(\theta)   \theta
  \\
  0
 \end{bmatrix}
\end{aligned} 
\label{e:zapzeroODEnewFast}
\end{equation}
This is a linear system with constant input.   It is stable because any eigenvalue $\lambda$ of the state matrix solves the equation
\[
\lambda^2  + \lambda +  \mu_+   =0
\]
for some eigenvalue $\mu_+>0$ of the positive definite matrix $A(\theta)  M  A(\theta) ^\transpose$.    Any solution to this equation lies in the strict left half plane of $\Co$.

The equilibrium $(w^\theta_\infty ; z^\theta_\infty)$ of \eqref{e:zapzeroODEnewFast}
 satisfies
\begin{equation}   
    - A(\theta) \{  \theta +z^\theta_\infty \} +  \barf(\theta)   =0 \,, \qquad   
       M  A(\theta) ^\transpose  w^\theta_\infty -   z^\theta_\infty  =0 \, .
\label{e:zapzeroEqui1}
\end{equation}

 \wham{2.} The equilibrium  for \eqref{e:zapzeroODEnewFast} is substituted into the dynamics for the slow variable to obtain the approximation $x_t\approx \odestate_t $ with
 \[
 \ddt x_t   = - [  \theta + z_\infty^\theta]\big|_{\theta = x_t}
\]
 The equilibrium equations \eqref{e:zapzeroEqui1} imply that $  \theta +  z^\theta_\infty = [A(\theta)]^{-1}  \barf(\theta)  $ for all $\theta$, so that we recover the Newton-Raphson flow,  
 $\ddt x_t   =    -  [A(x_t)]^{-1}  \barf(x_t)  $.

\begin{subequations}

\wham{SA Translation}

The 2023 version of the \textbf{Zap Zero SA} algorithm is defined by the $3d$-dimensional recursion motivated by the ODE \eqref{e:zapzeroODEnew}.
For initialization $\theta_0,w_0,z_0\,\in\Re^d  $, obtain the sequence of estimates recursively:  
\begin{align} 
\theta_{n+1} &= \theta_n  -  \alpha_{n+1} \bigl[    \theta_n +   L_{n+1}  w_n   \bigr]
\label{e:ZapZeroa} 
\\
w_{n+1} &= w_n - \stepf_{n+1}     \bigl[     A_{n+1} \{  \theta_n + z_n \} -  f_{n+1}(\theta_n)   \bigr]
\label{e:ZapZerow} 
\\
z_{n+1} &= z_n - \stepf_{n+1}     \bigl[        z_n  -  L_{n+1} w_n  \bigr]   && 
  L_{n+1} = M  A_{n+1} ^\transpose
\label{e:ZapZeroz} 
\end{align}
where as above $    A_{n+1} \eqdef \partial_\theta f_{n+1}(\theta_n)  $.
 The   two gain sequences  $\{\alpha_n\}$ and $\{\stepf_n\}$   satisfy  \eqref{e:gammaalpha}.
 
 \label{e:ZapZero}
  \end{subequations}
 
 \wham{Why is there a need for dimension \boldmath{$3d$}?}    Theory predicts that $ z_n \approx M A(\theta_n)^\transpose w_n$ for large $n$,  which motivates elimination of $\{z_n\}$ to obtain, 
 \begin{align*}
\theta_{n+1} &= \theta_n  -  \alpha_{n+1} \bigl[    \theta_n +   M \haA_{n+1}^\transpose  w_n   \bigr]
\\
w_{n+1} &= w_n - \stepf_{n+1}     \bigl[     \haA_{n+1} \{  \theta_n + M \haA_{n+1}^\transpose  w_n \} -  f_{n+1}(\theta_n)   \bigr]
\end{align*}
A third recursion is required to construct the matrix sequence $\{  \haA_{n+1} \}$,  which is why  \eqref{e:ZapZero} is a much simpler recursion.

 The assumptions in the following are adapted from \cite{chedevbusmey20b} in their treatment of Zap SA.

\begin{theorem}
\label[theorem]{t:ZapZero}
Suppose that the following hold:

\whamrm{(i)}  Assumptions (SA1) and (SA2).

\whamrm{(ii)}  
The derivative \eqref{e:Atheta}
is a continuous function of $\theta$,  satisfying $\det(A(\theta)) \neq 0$ for all $\theta$.

\whamrm{(iii)}  The function $\| \barf(\theta) \|$ is coercive:  $\displaystyle \lim_{\|\theta\| \to\infty}  \| \barf(\theta) \| = \infty$.

 Then, there is a unique solution to $\barf(\theta^*) = \Zero$,  and the Zap Zero algorithm  \eqref{e:ZapZero}  is convergent for each initial condition:
 \[
\lim_{n\to\infty } \theta_n =  \theta^*
\,,
\quad
 \lim_{n\to\infty } w_n =   w^*_\infty
\,, 
\quad
 \lim_{n\to\infty } z_n =   z^*_\infty  \qquad a.s.,
\]
with $w^*_\infty = w^\theta_\infty$,  $z^*_\infty = z^\theta_\infty$ evaluated at $\theta = \theta^*$. 
\end{theorem}

\begin{proof}  
As in the Newton-Raphson flow,
the coerciveness assumption is imposed to ensure that convergence  $ \barf(\odestate_t)  \to 0$ as $t\to\infty$, implies that $\{\odestate_t\}$ is bounded.
 We then conclude that any limit point is a root of $\barf$, and uniqueness of $\theta^*$ quickly follows.  
Convergence of  \eqref{e:ZapZero} then
 follows from standard theory of two time-scale stochastic approximation  \cite{bor20a}.
\end{proof}

Many generalizations are possible.  In particular,  it is shown in  \cite{chedevbusmey20b} that invertibility of $A(\theta)$ for all $\theta$ is not required, but may be replaced with the following:  it is assumed that     $ A(\theta) \barf(\theta) = \Zero$ holds  only if $\barf(\theta) = \Zero$.    This justifies   a modification of the Newton-Raphson flow in which 
$A(\odestate_t)^{-1}  \barf(\odestate_t) $  uses a pseudo-inverse when the matrix is not invertible.   It is also shown in this prior work that $A$ need not be continuous in applications to Q-learning.  Extending these generalizations to Zap Zero is a topic for research, but is not expected to be a big challenge.

Unfortunately we do not yet know how to verify all of the assumptions of \Cref{t:ZapZero} for Q-learning, in which  $\barf$ defined in \eqref{e:ProjBE},
even in the relaxed form described in the previous paragraph.  
This means the existence of a solution to the projected Bellman equation remains an open topic of research when using oblivious training.

\section{Stability with Optimism}
\label{s:Qstable}

The theory surveyed in the preceding section imposed oblivious training.   In the case of Watkins' Q-learning this assumption was imposed in part for 
historical reasons, though we will see that the analysis is somewhat more complex when we consider parameter dependent policies.  The technical challenges for Zap Q-learning are far more interesting because the definition of the linearization $A(\theta)$ is not obvious.
See the conclusions for further discussion.

We begin with a motivating example.

\begin{figure}[h]
\centering
 \includegraphics[width=0.5\textwidth]{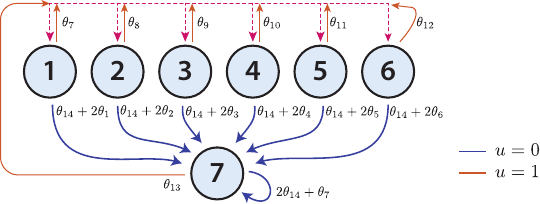} 
 
 \caption{Baird's star example.  }
 \label{f:BairdsStar}
 \end{figure}

\def\whamb{\wham{$\bullet$}}

\subsection{Baird's star example}
\label{s:BairdNumerics}

There are seven states $\state=\{1,\dots,7\}$ and two actions $\ustate=\{0,1\}$,  in which $X_{k+1} = 7$ with probability one whenever $U_k=0$.
If $U_k=1$ then $X_{k+1}$ is uniformly distributed over states $\{ 1,\dots, 6\}$.

In  \cite{bai95} it is assumed the cost is identically zero.   We take $c(x,u) = 0$ if $x\le 6$ and $c(7,u)=-10$ (independent of $u$).

The Q-function is linearly parameterized with dimension $d=14$.  
A schematic is shown in \Cref{f:BairdsStar}, adapted from  \cite{bai95}, indicating   the following values of $Q^\theta(x,u) $:
\whamb    $ \theta_{14} + 2\theta_x$ when $(x,u) = (x,0)$,  with $1\le x\le 6$.  
\whamb    $ 2\theta_{14} +\theta_7$ when $(x,u) = (7,0)$.
\whamb    $\theta_{13} $ when $(x,u) = (7,1)$.
\whamb  $\theta_{6+x}$  when $(x,u) = (x,1)$, with $1\le x\le 6$.

We refer the reader to the source \cite{bai95} (the final page contains a full description of the model considered in the experiments surveyed here),  and   \cite{sutbar18} for a fuller discussion.

\begin{figure}[h]
\centering
	\includegraphics[width=0.7\hsize]{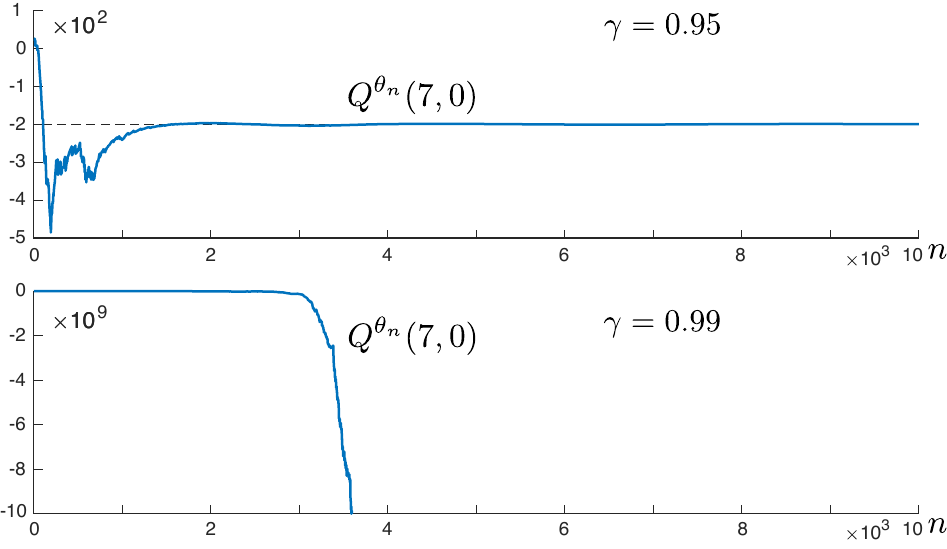}

\caption{Evolution of the   Q-function  approximations for two values of discount factor, and using an $\epsy$-greedy policy with common value of $\epsy=0.5$. }
\label[figure]{f:TwoDiscBaird}
\end{figure}

\smallskip

The following oblivious policy is considered in \cite{bai95}:   $U_k = 0$ with probability $1/7$, and otherwise $U_k=1$.    
While  well-motivated  from the point of view of exploration,   it was shown that the parameter estimates diverge when the discount factor is sufficiently large.

\Cref{f:TwoDiscBaird} shows trajectories from the Q-learning algorithm \eqref{e:Qlambda}
 with an $\epsy$-greedy policy using $\epsy=0.5$.  
The ideal behavior is that $Q^{\theta_n}(x,u) \to Q^\star
(x,u) = -10/(1-\disc)$ as $n\to\infty$ when
$(x,u) =(7,0)$.   The figure shows convergence when $\disc= 0.95$, but the parameters are divergent with discount factor $\disc=0.99$.

With the larger discount factor we obtain stability when using a smaller value of $\epsy>0$.  \Cref{f:BigDiscBaird} shows typical results for three small values.
The dashed line indicates $Q^*(7,0)$.

The step-size sequence was taken to be $\alpha_n = \min(\bar\alpha,  g / n^{\rho})$  using $ g = 1/(1-\disc)$,   $\rho=0.85 $,   and $\bar\alpha = 0.1$
in each run.    See \Cref{s:BairdCode} for Matlab code.

	 \begin{figure}[h]
\centering
	\includegraphics[width=0.7\hsize]{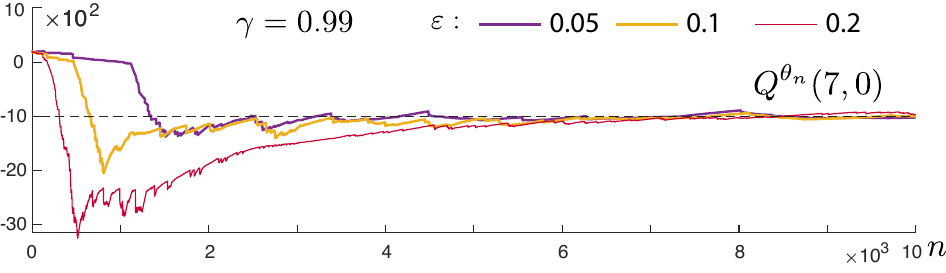}

\caption{Evolution of the   Q-function  approximations when using an $\epsy$-greedy policy.   Convergence holds when $\epsy>0$ is sufficiently small. }
\label[figure]{f:BigDiscBaird}
\end{figure}

See  \cite{devmey22} for an explanation for slow convergence with a large discount factor,  and    for explanation of the choice $ g = 1/(1-\disc)$.    This prior work is based on consideration of the linearization matrix $A^* = A(\theta^*)$  (see \Cref{t:Aproperties1}   for a representation).
\Cref{f:haA} shows a plot of the maximum real part of $A^*$ as a function of $\epsy>0$,   with $A^*$ estimated via Monte-Carlo.   
For larger values of $\epsy>0$ we see that $A^*$ is not Hurwitz for the three choices of discount factor.   
There is also trouble for very small $\epsy>0$:    The discussion following \Cref{t:chedevborkonmey21} \textit{suggests} that the asymptotic covariance will be very large when  $\max(\text{Real}\, \lambda(A^*))$ is close to zero,  but   the covariance $\Sigma_{\Delta}^* $ must also be considered to make any conclusions.

\begin{figure}[h]
\centering
		\includegraphics[width=0.7\hsize]{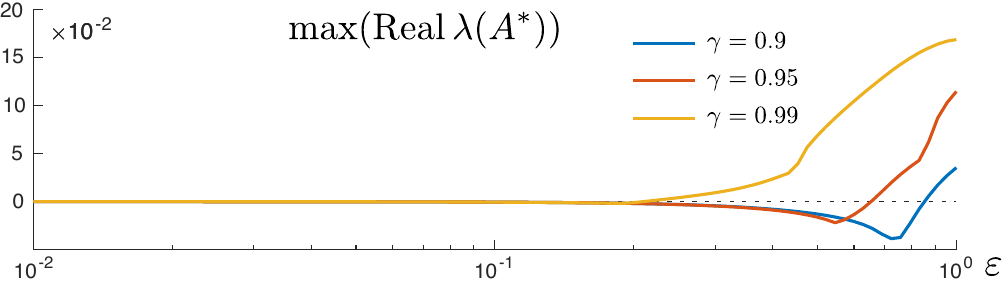}

\caption{The maximum eigenvalue of $A^*$ as a function of $\epsy$.    The matrix is   Hurwitz for sufficiently small
$\epsy>0$,   but some eigenvalues approach zero with vanishing $\epsy$.}
\label[figure]{f:haA}
 \end{figure}
	
\clearpage

\subsection{Stability with linear function approximation}

The main result of this section shows how exploration using a policy of the form \eqref{e:epsyGreedy} encourages stability of 
the Q-learning algorithm \eqref{e:Qlambda}.  Analysis requires the family of autocorrelation matrices,
\begin{subequations}%
\begin{align}
 \Rtheta(\theta) 
 & = 
 \Expect_{\uppi_\theta }\bigl[ \psi(X_{n}, \fee^\theta(X_n))  \psi(X_{n}, \fee^\theta(X_n))^\transpose\bigr]  
 \qquad
   \REXP(\theta)   = 
 \Expect_{\uppi_\theta }\bigl[  \psi(X_{n}, \EXP_n)    \psi(X_{n}, \EXP_n )^\transpose\bigr]  
 \label{e:RthetaEXP}
 \\[.5em]
 R(\theta)  &= \Expect_{\uppi_\theta }[ \psisub{n}\psisub{n}^\transpose]   =   (1-\epsy)   \Rtheta(\theta)  + \epsy \REXP (\theta) 
  \label{e:Rzero}
\end{align}
The expectations are  in steady-state, with stationary pmf $\uppi_\theta$  induced by the randomized stationary policy with fixed parameter.

A special case is considered in the assumptions, in which we take $\epsy=1$, and  the randomized policy is then denoted  $ \feex^{\tEXP}$, giving 
 $ \feex^{\tEXP}(\varble\mid x) = \upnu_\tEXP(u)$  for all $x,u$.   The  (assumed unique) invariant pmf is denoted   $\uppi_\tEXP$, and
the autocorrelation matrix 
\begin{equation}
\REXP = \Expect_{\uppi_\tEXP}  \bigl[  \psi(X_{n}, \EXP_n)    \psi(X_{n}, \EXP_n )^\transpose\bigr]  
   \,, 
\qquad
\textit{using $U_k = \EXP_k$ for all $k$.}
\label{e:RreallyOblivious}
\end{equation}

\label{e:RRR}
\end{subequations}

\begin{subequations}

The following assumptions are required in the main results of this section:
   \begin{align}
\parbox{.75\hsize}{\raggedright
The randomized policy $ \feex^{\tEXP}$ gives rise to a uni-chain Markov chain,  with unique invariant pmf $\uppi_\tEXP$, 
and    the autocorrelation matrix $\REXP$ defined in \eqref{e:RreallyOblivious} is positive definite.
}
\label{e:RankCondition}
 \\
 \parbox{.75\hsize}{\raggedright
 The inverse temperature
  $\kappa_\theta$ is twice continuously differentiable ($C^2$) in $\theta$, and the first and second derivatives of   $\kappa_\theta$ are  continuous and bounded.
  }
\end{align}

\label{e:TamedAssumptions}

\end{subequations}

 We also require small $\epsy>0$ in specification of the policies.    Denote
\begin{equation}
\epsy_\disc\eqdef 
		\frac{ (1-\disc)^2}{ (1-\disc)^2 + \disc^2}
 \label{e:epsydisc}
\end{equation}

\begin{theorem}
\label[theorem]{t:Qstable} 
Consider
the Q-learning algorithm \eqref{e:Qlambda} with linear function approximation,  and training policy \eqref{e:epsyGreedy} defined using the tamed Gibbs policy \eqref{e:epsyGreedySoftGibbTamed}.   Suppose moreover that \eqref{e:TamedAssumptions} holds.

Then, for any $\epsy \in (0, \epsy_\disc)$ there is $\kappa_{\epsy,\disc}>0$ for which
 the following hold using the  $(\epsy,\kappa_0)$-tamed Gibbs policy, using $\kappa_0\ge  \kappa_{\epsy,\disc}$:
 
\whamrm{(i)}  The parameter estimates $\{\theta_n\}$ are bounded:   there is a fixed constant  $B_\tTheta$,
independent of  $\kappa_0\ge  \kappa_{\epsy,\disc}$,
 such that \eqref{e:ODEbdd} holds 
with probability one from each initial condition.

\whamrm{(ii)}  There exists at least one solution to the projected Bellman equation \eqref{e:PBE}.
\end{theorem}


See \Cref{s:epsyGreedyWorks} for an extension of (ii) to the $\epsy$-greedy policy.

To see why (i) is plausible, consider   an algorithm approximating \eqref{e:Qlambda}, in which the minimum defining $   \uQ^{\theta_n} (X_{n+1}) $ is replaced by substitution of the input used for training: 
\begin{equation}
\begin{aligned}
\theta_{n+1} &= \theta_n + \alpha_{n+1}    \tilTdiff_{n+1}  \elig_n \, .
   \\
    \tilTdiff_{n+1} & = 
 \csub{n}  -Q^{\theta_n}(X_n,U_n)   + \disc   Q^{\theta_n} (X_{n+1} , U^-_{n+1}) 
  \end{aligned} 
\label{e:QlambdaApprox}
\end{equation}
in which $U_{n+1}^-$ is obtained by sampling from $   \feex^{\theta} ( \varble \mid x)  $   using $x=X_{n+1}$ and $\theta=\theta_n$.
The use of a soft-minimum instead of the hard minimum $ \uQ^{\theta_n} (X_{n+1} )$
 is common in the RL literature  \cite{sutbar18}.

    Stability of the ODE@$\infty$ is then relatively easy, from which we obtain the following:

\begin{proposition}
\label[proposition]{t:approxStable}

Consider the recursion \eqref{e:QlambdaApprox} with linear function approximation,  and training policy \eqref{e:epsyGreedy} defined as the
$(\epsy,\kappa_0)$-tamed Gibbs policy  \eqref{e:epsyGreedySoftGibbTamed} with  $\epsy \in (0,1)$ and $\kappa_0>0$.  Suppose moreover that \eqref{e:TamedAssumptions} holds.
 Then,
 we obtain the conclusions of   \Cref{t:Qstable}:

\whamrm{(i)} The parameter estimates $\{\theta_n\}$ are bounded with probability one from each initial condition.

\whamrm{(ii)}  There exists at least one solution $\theta^*$ to  $  \barf(\theta^*) = \Zero$,   with $\barf$ the mean flow for  \eqref{e:QlambdaApprox}.
 
\end{proposition}

We proceed with the proof of \Cref{t:approxStable}.    The proof of \Cref{t:Qstable} is postponed to the Appendix.

The recursion \eqref{e:QlambdaApprox} is similar to the sequence of fixed point equations, 
\[
\theta_{n+1} = \theta_n + \alpha_{n+1}  \bigl[
 \csub{n}  -Q^{\theta_n}(X_n,U_n)   + \disc   Q^{\theta_n} (X_{n+1} , U_{n+1}) \bigr]
 \]
 In which $\{U_n\}$ is obtained   from the tamed Gibbs policy (so that $U_{n+1}$ depends on $\theta_{n+1}$).
Assuming there is a solution for each $n$,  this recursion and    \eqref{e:QlambdaApprox} share the same mean-flow vector fields.   
 
The sequence of fixed point equations may be expressed in a form similar to the TD(0) learning algorithm,
\begin{equation}
\theta_{n+1} = \theta_n + \alpha_{n+1} \bigl[  \psisub{n}   \csub{n}    -   \psisub{n} \{  \psisub{n}  -\disc   \psisub{n+1}  \}^\transpose \theta_n \bigr]   
\label{e:QlambdaApprox2}
\end{equation}
This motivates consideration of  the family of autocorrelation matrices
\[
R_k (\theta) = \Expect_{\uppi_\theta }[ \psisub{n+k}\psisub{n}^\transpose]  \,,\quad n,k\ge 0 \, ,
\] 
so that  $R_0 (\theta) = R (\theta)$ in the notation  \eqref{e:Rzero}.

The vector field for the mean flow associated with \eqref{e:QlambdaApprox} is Lipschitz continuous and has an attractive form.

\begin{lemma}
\label[lemma]{t:ODEsForQopti}
Under the assumptions of \Cref{t:approxStable},
\whamrm{(i)}
The vector field for the mean flow is  
\[
\barf(\theta) = A(\theta)\theta -b(\theta)     \,, \quad \textit{ in which $b(\theta) = -\Expect_{\uppi_\theta}[ \psisub{n}   \csub{n} ]$ and $A(\theta) = - R_0(\theta) + \disc R_{-1}(\theta)$.  }
\]

\whamrm{(ii)}
The limit defining $\barfinf$ in \eqref{e:barfinfty} exists and may be  expressed 
\[
\barfinf(\theta) = A_\infty(\theta)\theta \,,\qquad   \textit{where} \quad  A_\infty(\theta) = A(\theta/\|\theta\|)\quad \textit{for $\theta\neq \Zero$.}
\]
\end{lemma}

\begin{proof}
Identification of  $\barf$ follows immediately from \eqref{e:QlambdaApprox2} since $\theta$ is held fixed in the definition of the mean flow. 
The representation of the ODE@$\infty$ follows from structure of the policy highlighted in \eqref{e:LargeParPolicy}, which implies
\[
 \uppi_{r\theta} =   \uppi_{\theta} \,, \quad A(r\theta) = A(\theta) \,,\quad \textit{and}   \quad b(r\theta) = b(\theta)\qquad \textit{for all $r\ge 1$ when $\|\theta\|\ge 1$.}
 \]
\end{proof}

 \begin{lemma}
\label[lemma]{t:Rbdds}
Suppose that \eqref{e:RankCondition} holds.  Then,  for the recursion \eqref{e:QlambdaApprox}    there exists $\delta_\psi>0$, independent of $\theta$ such that
\[
\begin{aligned}
  R_0(\theta) & \ge \delta_\psi    I    &&  \textit{  for all $ \theta\in\Re^d $}
   \\
      \theta ^\transpose A(\theta ) \theta       &  \le  - (1-\disc)  \delta_\psi   &&  \textit{  for all $ \theta\in\Re^d $, $\|\theta\|\ge 1$.}
\end{aligned} 
\]
 \end{lemma}

\begin{proof}
  The proof of the lower bound on $ R_0(\theta) $  is identical to the proof of  \Cref{t:AbsCts_pmf} in the Appendix.
From \Cref{t:ODEsForQopti}~(i) we have for $ \theta \in\Re^d $ satisfying $ \|\theta \|\ge 1$,
\[
\theta ^\transpose A(\theta ) \theta   =  -   \theta ^\transpose R_0(\theta )   \theta  + \disc  \theta ^\transpose R_{-1}(\theta)\theta   \le -(1-\disc)   \theta ^\transpose R_0(\theta )   \theta    
			\le  -  (1-\disc)  \delta_\psi  \|\theta \|^2   
\] 
\end{proof}

\begin{proof}[Proof of \Cref{t:approxStable}]
 Let $V_1(\theta ) =\half \|\theta \|^2$ and apply \Cref{t:ODEsForQopti,t:Rbdds} to obtain,  whenever $\|\odestate_t\|\ge 1$,
\[
\ddt V_1(\odestate_t) = \odestate_t^\transpose \barf(\odestate_t) = \odestate_t^\transpose\{ A(\odestate_t) \odestate_t - b( \odestate_t ) \}
\le   -   \delta_1 \|\odestate_t \|^2     +   \| \odestate_t \|  \|b( \odestate_t )\|  
\]
with $\delta_1 =  (1-\disc)     \delta_\psi$.
This gives,  with $\barb=\sup_\theta \|b( \theta )\|   <\infty$,  
\[
\ddt V_1(\odestate_t)  \le  -  \half  \delta_1 \|\odestate_t \|^2    \,,  \qquad    \| \odestate_t \|  \ge \max(1 , 2 \barb)
\]
We then obtain (v4) using $V(\theta )=\sqrt{V(\theta )} = \|\theta \|$ for $  \| \theta  \|  \ge\max(1 , 2 \barb)$  (modified in a neighborhood of  the origin to impose the $C^1$ condition):
\[
\ddt V(\odestate_t)  \le  -  \delta_v V (\odestate_t )    \,,  \qquad    \| \odestate_t \|  \ge \max(1 , 2 \barb) \,,\quad \textit{with $\delta_v =    \delta_1/4$.   }
\]

Part (i) then follows from  \Cref{t:chedevborkonmey21}~(i) and part (ii) from
\Cref{t:ODEstableImpliesPBE}.
\end{proof}

\subsection{Implications to the \boldmath{$\epsy$}-greedy policy}
\label{s:epsyGreedyWorks}

A full analysis of  Q-learning using the  $\epsy$-greedy policy for training is beyond the scope of this paper due to discontinuity of the vector field.     We find here that \Cref{t:Qstable} admits a partial extension.      

We consider here the corresponding mean flow \eqref{e:barfQgen},   and also the algorithm with matrix gain, whose mean flow vector field is
\[
\barfzap(\theta) =   - \theta   +   [A(\theta) ]^{-1} b(\theta)   \,,  \quad \theta \in \feeUnique  
\] 
This defines the dynamics expected when using    Zap Q-learning based on \eqref{e:ZAPnonlinearSA}.

The set  $ \feeUnique$ defined in \eqref{e:Sunique} may be expressed as the disjoint union,
\[
\feeUnique  = \bigcup_i  \feeUnique_i
\]
in which each $\feeUnique_i$ is an open convex polyhedron,   with $\fee^\theta = \fee^{\theta'} $   for all $\theta,\theta'\in \feeUnique_i$.   
Consequently,  both $\barf$ and $\barfzap$ are constant on each set $\feeUnique_i$.

For each $\theta\in\Re^d$,  denote by $\Upphi^\theta $ the set of all randomized $Q^\theta$-greedy policies:  if $\feex \in \Upphi^\theta $ then
\[
\sum_u \feex (u \mid x) Q^\theta (x,u)  =   \uQ^\theta (x)\,,\qquad x\in\state\,.
\]
If $\theta \in \feeUnique$ then $\Upphi^\theta =\{ \fee^\theta\} $ is a singleton.

\begin{theorem}
\label[theorem]{t:greedyStable}
Suppose that \eqref{e:RankCondition}  holds.    Then, the following hold for the mean flows associated with the Q-learning algorithm with $\epsy$-greedy training, provided $0<\epsy < \epsy_\disc$:

\whamrm{(i)}  There exists $\theta^*\in\Re^d$  and  $\feex^* \in \Upphi^{\theta^*} $  such that $\barf(\theta^*) =\Zero$, with $\barf$ defined in \eqref{e:barfQgen} in which the expectation is taken in steady-state using $\uppi_{\theta^*}$ obtained from the  randomized policy,
\begin{equation}
\feex^{\theta^*}  (u \mid x)  = (1-\epsy)  \feex^* (u\mid x)  +  \epsy \upnu_\tEXP(u) 
\label{e:epsyGreedyTrue2}
\end{equation}

\whamrm{(ii)}  If  $\theta^*\in\feeUnique $  then $\theta^*$ is locally asymptotically stable for the mean flow with vector field $\barf$.  

\whamrm{(iii)}  If  $\theta^*\in\feeUnique_i $ for some $i$,  then $\theta^*$ is locally asymptotically stable for the mean flow  with vector field $\barfzap$,  with domain of attraction including all of $\feeUnique_i$. 
\end{theorem}

 \begin{proof}
The proof of (i) is contained in  \Cref{s:Qfinalproofs}.
 
If $ \barf(\theta^*) = 0$ with $\theta^* \in \feeUnique$, it then  follows from the definition of the vector field that $\theta^* =   [A(\theta^*) ]^{-1} b(\theta^*)$.  Consequently,  for $\theta$ in a neighborhood of $\theta^*$ contained in $\feeUnique$,
\[
\barf(\theta) =     A(\theta^*)  \bigl( \theta  -  \theta^*  \bigr)
\] 
See \Cref{t:QstableODEGreedy} for a proof that $   A(\theta^*) $ is Hurwitz, so that $\theta^*$ is locally asymptotically stable as claimed in (ii).

We have under the assumptions of (iii),
\[
\barfzap(\theta) =   - \theta   +  \theta^*   \,,  \quad \theta \in \feeUnique_i
\] 
If $\odestate_0 \in \feeUnique_i$ it follows that the solution to $\ddt \odestate_t  = \barfzap(\odestate_t)$ is given by
\[
\odestate_t = \theta^*     +  [\odestate_t - \theta^*] e^{-t}
\]
Convexity of $\feeUnique_i$   ensures that $\odestate_t \in \feeUnique_i$ for all $t$,  which completes the proof of (iii).
\end{proof}

\section{Conclusions and Thoughts for the Future}

This article began as a companion to the INFORMS APS lecture delivered by the author in June, 2023.   
The scope of the lecture and this article grew to include several significant new contributions,  which invite many avenues for future research:

\wham{1.}  The new Zap Zero SA algorithm  \eqref{e:ZapZero} is only one possible approach to approximate the Newton-Raphson flow.
There may be approaches based on momentum---it will be worthwhile revisiting the NESA algorithm of \cite{devbusmey19a}.    

\wham{2.}  We now know that the existence of a solution to the projected Bellman equation exists under mild conditions, most important of which involves the choice of policy for training.      
 
\wham{3.} The extension to average cost optimal control will be possible through consideration of \cite{aboberbor01}.    And for the discounted case, better algorithms and better bounds on $\epsy$ might be obtained by adopting relative Q-learning algorithms \cite{devmey22,CSRL}.

\wham{4.} 
We should consider other paradigms for algorithm design.   The recent approaches  \cite{mehmey09a,bascurkraneu21,mehmeyneulu21,lumehmeyneu22} 
 are  based on the linear programming formulation of optimal control due to of Manne, 1960 \cite{man60a}.

\wham{5.}   \textit{Challenge with Zap}.
Based on theory surrounding the Actor-Critic method,  we have for a policy of the form  \eqref{e:OptimisticGeneric},
\begin{equation}
  \partial_\theta \barf\, (\theta)
 = \Expect_{\uppi_\theta} [ A_n(\theta) ]  
 +
  \Expect_{\uppi_\theta} [   \haf_n\, (\theta) \Lambda_n(\theta)^\transpose ]   
\label{e:AthetaScore}
\end{equation}
in which $A_n(\theta)  =  \partial_\theta f_n\, (\theta)$.   The expectations are in steady state  under $\uppi_\theta$ (recall discussion surrounding \eqref{e:Ttheta}).

The second expectation involves the \textit{score function} associated with the randomized policy,
\[
 \Lambda_n(\theta)  =  \nabla_\theta \log  \feex^{\theta} (u\mid x)   \Big|_{  u=U_n  \,, \ x= X_n}
\]
The function $\haf_n$ solves a certain Poisson equation.   
If the   transition matrix $T_\theta$ is aperiodic, then for a stationary realization 
of $\{ X_n,U_n : n\ge 0\}$    we have
\[
  \Expect_{\uppi_\theta} [   \haf_n (\theta) \Lambda_n(\theta)^\transpose ]   
=
\sum_{k=0}^\infty    \Expect_{\uppi_\theta} [   [f_{n-k} (\theta) -\barf(\theta)] \Lambda_n(\theta)^\transpose ]   
\]
Based on this representation we can obtain unbiased   estimates of $  \partial_\theta \barf\, (\theta_n)$   by adopting concepts from actor-critic algorithms.
See \cite[Ch.~10]{CSRL} for a survey in the style of this paper.

Analysis of the resulting algorithms will be considered in future research.

\clearpage

\appendix

  \clearpage
  
  \centerline{\LARGE\bf Appendix}
    
 \section{Stability with Optimism} 
 
 This section concerns analysis of Q-learning with optimistic training, so that the input is defined by a randomized policy $ 
 \feex^{\theta}$.    When $\theta$ is frozen,  so that $U_k \sim   \feex^{\theta}(\varble\mid X_k)$ for each $k$, 
  then $\bfmX$ is a time homogeneous Markov chain with transition matrix,   
  \begin{equation}
P_\theta(x,x')  \eqdef    \sum_u   \feex^{\theta} (u\mid x) P_u(x,x')    \,, \qquad x, x'  \in  \state\, .
\label{e:Ptheta}
\end{equation}
Recall that in this case the joint process $ \{ (X_k,U_k)  : k\ge 0\}$ is also Markovian, with transition matrix given in \eqref{e:Ttheta}.
We maintain the notation $\uppi_\theta$ for the unique invariant pmf for $T_\theta$.

Of course,  the parameter $\theta$ is never frozen in any algorithm.  The transition matrices $P_\theta$ and $T_\theta$ are introduced for analysis. 
 
 It is assumed that  the function class is linear,  $\{Q^\theta = \theta^\transpose \psi : \theta\in\Re^d\}$ with $\psi \colon\state\times \ustate\to\Re^d$.

\subsection{A truly oblivious policy}

We require structure of the truly oblivious policy defined by $U_k \equiv \EXP_k$  in the definition of $\REXP$ in \eqref{e:RreallyOblivious}.     
The transition matrix for the joint process $ \{ (X_k,U_k)  : k\ge 0\}$ can be obtained from \eqref{e:Ttheta}, and is denoted
\[
T_\tEXP(z,z')  =    P_u(x,x')  \upnu_\tEXP(u')  \,, \qquad  z=(x,u)\, ,\  z'=(x',u')  \in \state\times\ustate\, .
\]
The invariance equation $\uppi_\tEXP(z') = \sum_z \uppi_\tEXP(z) T_\theta(z,z') $ implies that the invariant pmf is product form:
\[
\uppi_\tEXP(z')   =   \upmu_\tEXP(x')   \upnu_\tEXP(u') \,,\qquad z'=(x',u')  \in \state\times\ustate\, .
\]
 in which $ \upmu_\tEXP(x')  = \sum_u  \uppi_\tEXP(x',u)$ is the steady-state marginal distribution of $\bfmX$ under this policy.    
 Similar notation is adopted for each of the invariant pmfs,
 \[
  \upmu_\theta(x)  = \sum_u  \uppi_\theta (x,u)  \,,  \qquad  x\in\state\,, \ \theta\in\Re^d\, .
 \]
 These are the invariant pmfs for $\{ P_\theta \}$ appearing in \eqref{e:Ptheta}.

\def\deltaWlowerDom{\delta_\bullet}
 
 \begin{lemma}
\label[lemma]{t:AbsCts_pmf}
Suppose that the Markov chain with transition matrix $T_\tEXP$ is uni-chain, so that $\uppi_\tEXP$ is the unique invariant pmf.   
Consider any one of the three choices of $\{\clU_k\}$ used in \eqref{e:epsyGreedy}  
with $\epsy<1$ and any choice of $\kappa$ in the case of \eqref{e:epsyGreedySoftGibb} or $\{\kappa_\theta\}$ in the case of  \eqref{e:epsyGreedySoftGibbTamed}.   The following conclusions then hold:   

\whamrm{(i)}  $T_\theta$ is also uni-chain, so that $\uppi_\theta$ is unique for any $\theta$.

\whamrm{(ii)}    There is a constant $\deltaWlowerDom  >0$ such that  $\uppi_\theta(z)  \ge  \deltaWlowerDom   \uppi_\tEXP(z) $ for all $z$ and $\theta$.
The constant $\deltaWlowerDom  $   may depend on the  policy parameters, but not $\theta$.


\whamrm{(iii)}   $\uppi_\theta (x,u)  \ge  \epsy    \upmu_\theta(x)  \upnu_\tEXP(u)$ for all $x,u$, and   $\theta$.

\whamrm{(iv)}  $
   \REXP(\theta)  \ge \deltaWlowerDom      \REXP$  for all $ \theta\in\Re^d$.
\end{lemma}

 \begin{proof} 
 Let $\state_0$ denote the support of $\upmu_\tEXP$ and  $\ustate_0$ denote the support of $\upnu_\tEXP$.    
 The uni-chain assumption is equivalent to the following reachability criterion: 
 	 there is $N\ge 1$  and $\delta_N>0$ such that   
\[
 \sum_{k=1}^N  T^k_\tEXP(z,z')  \ge \delta_N\,,\qquad  
 	\textit{for any $z \in \state\times\ustate$,
  			and  $ z' \in \state_0\times\ustate_0$,}
 \]
 with $ T^k_\tEXP$ the $k$-step transition matrix. 
 This is a version of Doeblin's minorization condition that implies uniform ergodicity when the chain is aperiodic \cite{MT}.

In view of \eqref{e:epsyGreedy} we have for any $\theta$,
 \[
 \sum_{k=1}^N  T^k_\theta(z,z')  \ge 
 \sum_{k=1}^N  \epsy^k  T^k_\tEXP(z,z')  \ge 
 \epsy^N  \delta_N
 		\,,\qquad  
 	\textit{for any $z \in \state\times\ustate$,
  			and  $ z' \in \state_0\times\ustate_0$,}
 \]
 Hence the family $\{T_\theta : \theta\in\Re^d \}$ satisfies a uniform Doeblin minorization.   In particular, each transition matrix is uni-chain, which establishes (i).   
 
 Part (ii) follows from the bounds above and invariance:
 \[
 \uppi_\theta(z')  =   \sum_z    \uppi_\theta(z)  \Bigl(  \frac{1}{N}  \sum_{k=1}^N  T^k_\theta(z,z')   \Bigr)
 \ge  
  \frac{1}{N} 
 \epsy^N
 \delta_N\,,\qquad z' \in \state_0\times\ustate_0
 \]

Part (iii) also follows from invariance in the following one-step form:   we have from \eqref{e:Ttheta},
and using the bound $ \feex^{\theta} (u' \mid x')  \ge \epsy  \upnu_\tEXP(u') $,   
\[
\begin{aligned}
 \uppi_\theta(z')  =   \sum_z    \uppi_\theta(z)  T_\theta(z,z')   
 & =    \sum_{x,u}    \uppi_\theta(x,u)  P_u(x,x') \feex^{\theta} (u' \mid x')   
 \\
 &
 						\ge   \epsy    \sum_{x,u}   \uppi_\theta(x,u)     P_u(x,x')   \upnu_\tEXP(u') 
\\
&						=   \epsy    \upmu_\theta(x')  \upnu_\tEXP(u')
\end{aligned} 
\]

 For  part (iv)  consider the definition  \eqref{e:RthetaEXP}, which gives
 \[
  \REXP(\theta)   =   \sum_{x,u}  \upmu_\theta(x) \upnu_\tEXP(u)
  \psi(x,u)  \psi(x,u)  ^\transpose   
\]
Applying (ii) gives $ \upmu_\theta(x) \ge   \deltaWlowerDom  \upmu_\tEXP(x) $ for all $x$, and hence
 the desired bound:
 \[
  \REXP(\theta)    
  \ge  
  \deltaWlowerDom   \sum_{x,u}  \upmu_\tEXP(x) \upnu_\tEXP(u)
  \psi(x,u)  \psi(x,u)  ^\transpose    =   \deltaWlowerDom   \REXP \, 
\]
  \end{proof}
  

\subsection{Mean flow for the \boldmath{$\epsy$}-greedy policy}


 In this subsection the input is chosen to be the $\epsy$-greedy policy \eqref{e:epsyGreedy}. The motivation is in part the fact that establishing stability of the ODE@$\infty$   in this case is far easier than the tamed Gibbs approximation. 
 
The transition matrix  \eqref{e:Ttheta} becomes  
\begin{equation}
T_\theta(z,z') =            P_u(x,x')    \bigl\{  (1-\epsy)  \ind\{ u' = \phi^\theta(x') \}  + \epsy \nu^\tEXP(u')  \bigr\} 
					\,, \qquad  z=(x,u)\, ,\  z'=(x',u')  \in \state\times\ustate\, .
\label{e:PthetaA1a}
\end{equation}
The family $\{ T_\theta : \theta\in\Re^d\}$ is finite because there are only a finite number of deterministic stationary policies;   
it takes on a constant value on each connected component of $\feeUnique$ (recall \eqref{e:Sunique}).

Compact representations of $f$ and $\barf$ are obtained with additional notation.   For $n\ge 0$ denote
\begin{equation}
\begin{aligned}
 \psisub{n}^\tTheta &= \psi(X_{n}, \fee^{\theta_n}(X_n))    \quad
&
 \psisub{n}^\tEXP &= \psi(X_{n}, \EXP_n)   
\\
\ctheta_n & = c(X_{n}, \fee^{\theta_n}(X_n)) 
 &
\cEXP_n &= c(X_{n}, \EXP_n)  
\end{aligned}  
\label{e:CompactNotation}
\end{equation}
We have under  the $\epsy$-greedy policy (\ref{e:epsyGreedy}, \ref{e:epsyGreedyTrue}),
\begin{equation}
\begin{aligned}
 f_{n+1}(\theta_n) 
 =  (1 - B_n) &  \bigl( 
     \ctheta_n   +         \bigl[ \disc \psisub{n+1}^\tTheta      - \psisub{n}^\tTheta  \bigr]^\transpose \theta_n  \bigr) \psisub{n}^\tTheta 
   \\
{}  +B_n & \bigl( 
     \cEXP_n   +         \bigl[  \disc\psisub{n+1}^\tTheta      - \psi^\tEXP_n  \bigr]^\transpose \theta_n  \bigr)  \psisub{n}^\tEXP
\end{aligned} 
\label{e:SAQ}
\end{equation}

\begin{lemma}
\label[lemma]{t:psipsiThetaGreedy}
$\displaystyle  \Expect_{\uppi_\theta}   \bigl[   \psisub{n}   \{  \psisub{n+1}^\tTheta  \} ^\transpose  \bigr]  =     R_{-1}(\theta)  + \epsy D(\theta)$,
in which
\begin{equation}
	  D(\theta) =  \Expect_{\uppi_\theta}   \bigl[   \psisub{n}   \{  \psisub{n+1}^\tTheta  -  \psisub{n+1}^\tEXP \} ^\transpose  \bigr]  
 \label{e:Dgreedy}
\end{equation}
\end{lemma}

\begin{proof}
Starting with the definition $    R_{-1}(\theta)  =   \Expect_{\uppi_\theta}   \bigl[   \psisub{n}   \{  \psisub{n+1}  \} ^\transpose  \bigr] $, we have under the $\epsy$-greedy policy,
\[
\begin{aligned}
    R_{-1}(\theta)  &=  (1-\epsy)   \Expect_{\uppi_\theta}   \bigl[   \psisub{n}   \{  \psisub{n+1}^\tTheta  \} ^\transpose  \bigr]   
					+\epsy  \Expect_{\uppi_\theta}   \bigl[   \psisub{n}   \{  \psisub{n+1}^\tEXP  \} ^\transpose  \bigr]  
					\\
					&= \Expect_{\uppi_\theta}   \bigl[   \psisub{n}   \{  \psisub{n+1}^\tTheta  \} ^\transpose  \bigr]  +  \epsy  
					\Expect_{\uppi_\theta}   \bigl[   \psisub{n}   \{  \psisub{n+1}^\tEXP - \psisub{n+1}^\tTheta  \} ^\transpose  \bigr]  
\end{aligned} 
\]
\end{proof}

\begin{subequations}

\begin{lemma}
\label[lemma]{t:Aproperties1}   
The vector fields for the mean flow and the ODE@$\infty$ for the $\epsy$-greedy policy are 
\begin{align}
  \barf(\theta) = A(\theta) \theta  &- b(\theta)  	\qquad\qquad
		  \barf_\infty(\theta) = A(\theta) \theta  
\label{e:epsy-greedy-flow}
\\[.5em]
\textit{ in which  }
\quad
A(\theta)  & =  -  \bigl[ R_0(\theta) - \disc R_{-1}(\theta) \bigr]  
					+ \epsy\disc D(\theta)     
\label{e:Agr}
   \\ 
b(\theta)  &=  (1-\epsy)\btheta(\theta)   +  \epsy  \bEXP(\theta)
\label{e:bgr}
\end{align}  
		$\btheta(\theta)  = - \Expect_{\uppi_\theta}[    \psisub{n}^\tTheta   \csub{n}^\tTheta  ]$   and 
		$  \bEXP(\theta) = -\Expect_{\uppi_\theta}[ \psisub{n}^\tEXP  c(X_{n}, \EXP_n)   ]$.
\end{lemma}

\begin{proof}   The representation \eqref{e:SAQ} is equivalently expressed  $ f_{n+1} (\theta_n)  =   A_{n+1}  \theta_n  - b_{n+1}$, in which  
\[
\begin{aligned}
A_{n+1}  &=     \psisub{n}    \bigl[\disc  \psisub{n+1}^\tTheta      - \psisub{n}  \bigr]^\transpose     
     \\
b_{n+1}& = (1 - B_n)   \psisub{n}^\tTheta    c(X_{n}, \fee^{\theta}(X_n))      +B_n  \psisub{n}^\tEXP   c(X_n, \EXP_n)   
\end{aligned} 
\]
The expression for $b(\theta)$ in the expression $\barf(\theta) = \Expect_{\uppi_\theta}[f_{n+1}(\theta ) = A(\theta) \theta - b(\theta) $ is immediate.

We have $A(\theta)   =  - R_0(\theta)    +\disc  \Expect_{\uppi_\theta}   \bigl[   \psisub{n}   \{  \psisub{n+1}^\tTheta  \} ^\transpose  \bigr] $,
so that  \eqref{e:Agr} follows from  \Cref{t:psipsiThetaGreedy}.

 The expression for $\barf_\infty$ follows from the fact that $A$ and $b$ are invariant under positive scaling of their arguments:   $A(r\theta) = A(\theta)$ and $b(r\theta) = b(\theta)$ for any $\theta$ and $r>0$.
 \end{proof}

\label{e:A+b}
\end{subequations}

 The mean flow \eqref{e:ODESA1} is a differential inclusion because the vector field $\barf$ is not continuous.

The form of the expression for $A(\theta)$ in \eqref{e:Agr} is intended to evoke the similar formula \eqref{t:ODEsForQopti}  obtained for \eqref{e:QlambdaApprox2}.

The following conclusions are based on arguments similar to what is used to obtain stability of on-policy TD-learning   \cite{tsivan97}.    
Recall the definition  \eqref{e:epsydisc}: $\epsy_\disc\eqdef   (1-\disc)^2  / [ (1-\disc)^2 + \disc^2 ]$.

\begin{proposition}
\label[proposition]{t:QstableODEGreedy}
If $\epsy< \epsy_\disc$, then there is   is $\beta_\epsy>0$ such that 
$v^\transpose A(\theta)v  \le  -\beta_\epsy \| v\|^2$  for each $v,\theta\in\Re^d$.
\end{proposition}

 \begin{proof} 
 Applying \Cref{t:Aproperties1}   gives for any $v,\theta$,
 \begin{equation}
v^\transpose A(\theta) v    \le   -  (1-\gamma)  v^\transpose  R_0(\theta)   v
					+
							  \epsy \gamma  v^\transpose D (\theta) v
\label{e:vAtheta1}
\end{equation}
The inequality follows from the bound $ v^\transpose  R_k(\theta)   v  \le  v^\transpose  R_0(\theta)   v$, valid for any $k$. 

We are left to bound the term involving $D$.
Write 
\[
v^\transpose    D(\theta) v =  \Expect_{\uppi_\theta} \bigl[  
	 (v^\transpose \psisub{n} )  ( v^\transpose \psisub{n+1}^\tTheta ) 
	 			 \bigr] 
	-
 \Expect_{\uppi_\theta} \bigl[  	( v^\transpose \psisub{n} )(  v^\transpose\psisub{n+1}^\tEXP)    
			 \bigr]  
\]
Using the bound $xy\le \half [x^2 + y^2]$ for $x,y\in\Re$, we obtain for any $\delta_\tEXP, \delta_\tTheta >0$, 
\[
\begin{aligned}
\big|   \Expect_{\uppi_\theta} \bigl[    (v^\transpose \psisub{n} )  ( v^\transpose \psisub{n+1}^\tTheta )   \bigr]    & 
			 \le    \half \delta_\tTheta^{-1} v^\transpose R_0(\theta) v    +   \half \delta_\tTheta  v^\transpose R_0^\tTheta(\theta) v 
			 \\
\big|   \Expect_{\uppi_\theta} \bigl[  ( v^\transpose \psisub{n} )(  v^\transpose\psisub{n+1}^\tEXP)     \bigr]  \big| &  
			 \le    \half  \delta_\tEXP^{-1} v^\transpose R_0(\theta) v    +   \half \delta_\tEXP  v^\transpose R_0^\tEXP(\theta) v 
\end{aligned} 
\]
Recall from  \eqref{e:Rzero} that $R_0(\theta) = (1-\epsy) R_0^\tTheta(\theta)  + \epsy  R_0^\tEXP(\theta) $.    Set 
$\delta_\tEXP  =\epsy \eta$, $\delta_\tTheta = (1-\epsy ) \eta$,  with     $\eta>0$ to be chosen.   Then,
\[
\begin{aligned}
v^\transpose    D(\theta) v   & \le   \half \Bigl[  \bigl( \delta_\tEXP^{-1} +   \delta_\tTheta^{-1}  \bigr) v^\transpose R_0(\theta) v
					+ \delta_\tTheta  v^\transpose R_0^\tTheta(\theta) v
					+  \delta_\tEXP  v^\transpose R_0^\tEXP(\theta) v      \Bigr]
   \\
            & =    \half \Bigl[    \Bigl( \frac{1}{\epsy}  + \frac{1}{1- \epsy}  \Bigr) \frac{1}{\eta}      +   \eta  \Bigr]   v^\transpose R_0(\theta) v
\end{aligned} 
\]
Minimizing the right hand side over $\eta$ gives $ \eta^*_\epsy = \sqrt{ \epsy^{-1} + (1-  \epsy)^{-1} }$,  and on substitution,  
\[
v^\transpose    D(\theta) v   \le \eta^*_\epsy \,  v^\transpose R_0(\theta) v  
\]

Substitution into \eqref{e:vAtheta1} gives the final bound,
\[
v^\transpose A(\theta) v    \le  \bigl[ -  (1-\gamma)   +
							  \epsy \gamma  \eta^*_\epsy     \bigr]  v^\transpose  R_0(\theta)   v
 \]
The coefficient is negative for positive $\epsy$ if and only if $\epsy<\epsy_\disc$.   We obtain the desired bound with
\[
\beta_\epsy =    \bigl[  (1-\gamma)  -  \epsy \gamma  \eta^*_\epsy     \bigr] \min_\theta \lambdamin(R_0(\theta) )
\]
\Cref{t:AbsCts_pmf}
implies that the minimum is strictly positive. 
 \end{proof}

The extension of \Cref{t:QstableODEGreedy} to the tamed Gibbs policy requires approximations summarized in the next subsection.
 
 \subsection{Entropy and Gibbs bounds}

 Consider a single Gibbs   pmf on $\ustate$ with energy $E\colon\ustate\to\Re$ and  inverse temperature $\kappa>0$:
 \[
 p_\kappa(u) =    \frac{1}{ \clZ_\kappa }  \exp(-\kappa E(u) )   \,, \qquad u\in\ustate\, . 
 \]
 The normalizing factor $\clZ_\kappa$ is commonly called the \textit{partition function}.    The entropy of $ p_\kappa$ is denoted
 \[
 H_\kappa = - \sum_u  p_\kappa(u) \log( p_\kappa(u) ) = \sum_u  p_\kappa(u) \bigl[ \kappa E(u)  + \log (\clZ_\kappa) \bigr]
 \]
 It is well known that bounds on entropy lead to bounds on the quality of the softmin approximation.
 
 Denote $\uE \eqdef \min_u E(u)$.
 
 \begin{lemma}
\label[lemma]{t:GibbsSoftMinBdd}
$\displaystyle\uE  \le \sum_u  p_\kappa(u)  E(u)    \le  \uE +\frac{1}{\kappa}  \log(|\ustate|)  $,  for any $\kappa>0$.   
\end{lemma}

\begin{proof}
The uniform distribution maximizes entropy, giving
\[
\sum_u  p_\kappa(u) \bigl[ \kappa E(u)  + \log (\clZ_\kappa) \bigr]   \le    \log(|\ustate|)
\]
The proof is completed on substituting the following bound for the log partition function:  
\[
 \log (\clZ_\kappa)   =  \log  \sum_u  \exp(-\kappa E(u) )   \ge   -\kappa \uE
 \]
\end{proof}

An implication of the lemma to the  policy  \eqref{e:epsyGreedySoftGibbTamed}:    for any initial distribution for $(X_0, U_0)$,     
\begin{equation}
 \uQ^\theta(X_{k+1})  
 \le 
\Expect \bigl[ Q^\theta(X_{k+1}, \clU_{k+1})   \mid X_0^{k+1}, U_0^k\bigr]
			 \le   \uQ^\theta(X_{k+1}) +  \frac{1}{\kappa_\theta}  \log(|\ustate|)   \,,\qquad  k\ge 0 \, .
\label{e:GibbsSoftMinImplication}
\end{equation}

 \subsection{Proof of  \Cref{t:Qstable,t:greedyStable}}
 \label{s:Qfinalproofs}
 
 The proof of  \Cref{t:Qstable} closely follows the proof of \Cref{t:QstableODEGreedy}.   We begin a companion to \Cref{t:psipsiThetaGreedy}:

\begin{lemma}
\label[lemma]{t:psipsiTheta}
We have for the $(\epsy,\kappa_0)$-tamed Gibbs policy,
\begin{subequations}%
\begin{align}
  \Expect_{\uppi_\theta}   \bigl[   \psisub{n}    \{  \psisub{n+1}^\tTheta  \} ^\transpose  \bigr]  
			=     R_{-1}(\theta) & + \epsy D(\theta)   +  (1-\epsy)  E(\theta)
\\[.5em]
\textit{in which} \quad
	  D(\theta) &=  \Expect_{\uppi_\theta}   \bigl[   \psisub{n}   \{  \psisub{n+1}^\tTheta  -  \psisub{n+1}^\tEXP \} ^\transpose  \bigr]  
 \label{e:DGibbs}
	  \\
	  E(\theta) &=  \Expect_{\uppi_\theta}   \bigl[   \psisub{n}   \{  \psisub{n+1}^\tTheta  -  \psisub{n+1}^\clU  \} ^\transpose  \bigr]  
 \label{e:EGibbs}
\end{align}
with $ \psisub{n+1}^\clU = \psi(X_{n+1}, \clU_{n+1})$.%
 \end{subequations}%
\end{lemma}

 We have a partial extension of \Cref{t:QstableODEGreedy}:
 
\begin{lemma}
\label[lemma]{t:QstableODEGibbs}
The following holds for the $(\epsy,\kappa_0)$-tamed Gibbs policy, subject to  \eqref{e:TamedAssumptions} and   $\epsy< \epsy_\disc$:  
 there is   is $\beta_\epsy>0$ such that 
$\theta^\transpose A(\theta)\theta  \le  -\beta_\epsy \| \theta\|^2$  
for all $\kappa_0>0$ sufficiently large, and all   $\| \theta \| \ge 1$.
\end{lemma}

\begin{proof}  
Applying \Cref{t:psipsiTheta} to \eqref{e:barfQgen}, and following the same steps as in the proof of  \Cref{t:QstableODEGreedy} we obtain
\[
\theta^\transpose A(\theta) \theta    \le  -\beta_\epsy^0  \,  \theta^\transpose  R_0(\theta)   \theta     +  \disc (1-\epsy)  \theta^\transpose E(\theta) \theta
 \]
 with  $\beta_\epsy^0 =  \bigl[  (1-\gamma) - \epsy \gamma  \eta^*_\epsy     \bigr] >0$,  with  $ \eta^*_\epsy = \sqrt{ \epsy^{-1} + (1-  \epsy)^{-1} }$.

 From the definition   \eqref{e:EGibbs} we have
 \[
 \theta^\transpose E(\theta) \theta  =  
   \Expect_{\uppi_\theta}   \bigl[  Q^\theta(X_n,U_n)    \{  \uQ^\theta(X_{n+1} )   - Q^\theta(X_{n+1} ,  \clU_{n+1})   \} \bigr]  
\]
Applying \eqref{e:GibbsSoftMinImplication}  and the expression for $\kappa_\theta$ in \eqref{e:normGibbsBdds},  we obtain for  $\| \theta \| \ge 1$,  
\[
\big | \theta^\transpose E(\theta) \theta \big |   
\le  \frac{ 1 }{\kappa_0}  \|\theta\|    \log(|\ustate|) 
   \Expect_{\uppi_\theta}   \bigl[  \big | Q^\theta(X_n,U_n)  \big | \bigr]   
\le  \frac{ 1 }{\kappa_0}  \|\theta\|^2   \log(|\ustate|)  \sqrt{\lambdamax}
\]
 with $\lambdamax $ the maximum over all $\theta$ of the maximum eigenvalue of $R_0(\theta)$.  Combining these bounds completes the proof. 
\end{proof}

\begin{proof}[Proof of  \Cref{t:Qstable}]
Precisely as in the proof of  \Cref{t:approxStable} we obtain a solution to (v4) using $V(\theta) = \|\theta \|$  (recall 	\eqref{e:ddt_bound_LyapfunTmp}),   which implies \eqref{e:ODEbdd} exactly as in the case when $\bfPhi$ is exogenous.  
  
The existence of $\theta^*$ follows from \Cref{t:ODEstableImpliesPBE}, exactly as in   the proof in \Cref{t:approxStable}
\end{proof}

\begin{proof}[Proof of  \Cref{t:greedyStable}]

Let $\theta^{\kappa_0} $ denote the solution to the projected Bellman equation for the $(\epsy,\kappa_0)$-tamed Gibbs policy,  in which  $\epsy<\epsy_\disc$ is fixed. 

Observe that in  \Cref{t:QstableODEGibbs}  we obtain a uniform bound over all large $\kappa_0$.     An examination of the proof of \Cref{t:ODEstableImpliesPBE} shows that there is a constant $b_\epsy$ such that $\|  \theta^{\kappa_0}  \|  \le b_\epsy$ for all sufficiently large $\kappa_0$.  

Hence we can find a subsequence $\kappa_0^n  \to \infty$ as $n\to \infty$,  for which the following limits exist: 
\[
\theta^*  =  \lim_{n\to\infty}  \theta^{\kappa_0^n}  \,, \quad
\uppi^* = \lim_{n\to\infty}  \uppi_n  \,,   
\]
in which $\uppi_n$ is the invariant pmf obtained from the policy using $ \theta^{\kappa_0^n}  $.

The invariant pmfs  have the form
\[
\uppi_n (x,u)  =   \upmu_n(x) \feex^n  (u \mid x)  
\]
with $ \feex^n  $ defined in \eqref{e:epsyGreedySoftGibbTamed}  using $\kappa_0^n$,   and $\upmu_n$ the first marginal of $\uppi_n$.
It follows that the limiting invariant pmf has the same structure,
\[
\uppi^*(x,u)  =   \upmu^*(x) \feex^{\theta^*}  (u \mid x)  
\]
Since $\kappa_0^n\uparrow \infty$,
convergence implies that  $\feex^{\theta^*} $ is of the form  \eqref{e:epsyGreedyTrue2} with $\feex^* \in \Upphi^{\theta^*} $.

Letting $\barf_n$ denote the vector field obtained using $ \theta^{\kappa_0^n}  $ we must have convergence for each $\theta$:
\[
\barf(\theta) =  \lim_{n\to\infty} \barf_n (\theta)  =  \Expect_{\uppi^*} [  \psisub{n}    \BelErr  (X_n,U_n;\theta)]\,,
\]
in which $U_n$ is defined using the randomized $\epsy$-greedy policy $ \feex^{\theta^*}  $,   and $ \BelErr$ defined in  \eqref{e:BE} is a continuous function of $\theta$.     Since  $\barf_n (\theta_n)   =0$ for each $n$,  we conclude that 
$\barf(\theta^*) = \Zero$ as desired.
\end{proof}

\subsection{Code for Baird's star example}
\label{s:BairdCode}

The basis vector corresponding to Baird's example is dimension 14, and defined as follows:
\[
\psi_i(x,u)    
=
\begin{cases}  
2 \ind\{x=i, u=0\}
              &   1\le i\le 6
\\
 \ind\{x=7, u=0\}
   + \ind\{x=1, u=1\}
              &   i=7
\\
 \ind\{x+6=i, u=1\}
              &   8\le i\le 13
 \\
 \ind\{u=0\} [1 +  \ind\{x=7\}   ]           &   i=14
\end{cases} 
\]
With the code below we find for a time horizon of $N=10^7$,  with $\epsy=0.05$ and $\gamma=0.99$, {\small
\[
\thetaPR_N=  [    
 -379;
    -379;
    -379;
    -379;
    -379;
    -379;
    -536;
    -802;
    -700;
    -620;
    -703;
    -699;
    -992;
    -234]
    \]}%
    where $\thetaPR_N$ was obtained by averaging over the final 90\%\ of the estimates.
    
    Estimation of $A^*$ was performed using this estimate of $\theta^*$.  
    
    The code allows for decaying, non-vanishing $\epsy$.  Fixed $\epsy>0$ was used in each experiment.
    
    \begin{verbatim}
% Baird's Counterexample Q-learning

rng(2)

num_states = 7;
num_actions = 2;

d=num_states*num_actions;  %in Baird's example

% Runlength
Hor = 1e4;

% Discount factor
disc = 0.99; 

epsyFinal = 0.05; % Epsilon-greedy exploration
% Option for decaying exploration
epsyAll=1:Hor;
epsyAll=epsyAll.^(-0.2);
epsyAll=max(epsyAll,epsyFinal);

% Define the reward vector, depending only on the state
R = [0; 0; 0; 0; 0; 0; 10];
C=-R;   %Stick to notation in my paper

LowestQ=C(7)/(1-disc);

% step-size pars:
g = 1/(1-disc); %
rho=0.85   ;

alphaAll=1:Hor;
alphaAll=g*alphaAll.^(-rho);

alphaAll=min(alphaAll,0.1);

haA=zeros(d,d);

% Initialize theta

thetaAll=zeros(d,Hor);
thetaNow=rand(d,1)/(1-disc);
thetaAll(:,1)= thetaNow;

Q71All=zeros(1,Hor);

state = randi(num_states); % Start in a random state
action = randi(num_actions)-1; % Start in a random input, values 0 or 1

psi=rand(d,1);

Cov=zeros(d,d);   %estimation of \Sigma_\psi for diagnostics

for n = 1:Hor

    alpha = alphaAll(n);

    beta=alpha^0.9;

    epsy=epsyAll(n);

    cost = C(state);

    psi = getpsi(state,action,d);

    Cov = Cov+ beta*((psi*psi'-Cov);

    state_new=getX(action) ;

    Q_now=thetaNow'*psi;

    [uQ_new,u_greedy] = uQ(state_new,thetaNow,d);

    upsi_new = getpsi(state_new,u_greedy,d);

    haA=haA+ beta*( (- psi + disc*upsi_new )*psi'- haA);

    TempDiff= cost - Q_now + disc*uQ_new;

    thetaNow=thetaNow+alpha*TempDiff*psi;

    thetaAll(:,n)=thetaNow;

    state = state_new;

    [~,u_greedy] = uQ(state,thetaNow,d);

    action=get_action(u_greedy,epsyFinal) ;

    psiBest = getpsi(7,0,d);
    Q71All(n)=thetaNow'*psiBest;
end


% Plots

figure(1)
plot(thetaAll(1,:))
hold on
plot(thetaAll(14,:))

hold off

figure(2)
plot(Q71All)
hold on
plot(LowestQ*ones(size(Q71All)),'--')

hold off
\end{verbatim}

\clearpage

\wham{Functions:}

\begin{verbatim}
function [psi_new] = getpsi(x,u,d)
%Baird basis

psi_new=zeros(d,1);

for i=1:6
    psi_new(i)= 2*(1-u)*(x==i);
end

psi_new(7)=(1-u)*(x==7)+ u*(x==1)  ;

for i=8:13
    psi_new(i) = u*((x+6)==i);
end

psi_new(14)= (1-u)*(1+(x==7));     

end

function [u_new] = get_action(u_greedy,epsy)

Greed= rand<(1-epsy);
u_obliv= (rand<0.67);
u_new=Greed*u_greedy + (1-Greed)*u_obliv;

end

function [x_new] = getX(u)
%Baird dynamics
if u==0
    x_new=7;
else
    x_new=randi(6);
end
end

function [uQ_new,u_greedy] = uQ(x,theta,d)
% Minimize Q-function approximation for given state and parameter
psi_zero = getpsi(x,0,d);
psi_one = getpsi(x,1,d);

[uQ_new,u_greedy]=min([theta'*psi_zero,theta'*psi_one]) ;
u_greedy= u_greedy-1;

end
\end{verbatim}

\clearpage

\wham{Estimation of $A^*$}

\begin{verbatim}
% Baird's Counterexample Q-learning - estimating A^*

num_states = 7;
num_actions = 2;

d=num_states*num_actions;  %in Baird's example

% Runlength
Hor = 1e6;

% Number of epsilons
M=50;

disc = 0.99; % Discount factor

epsyFinal = 0.6; % Epsilon-greedy exploration
epsyAll=1:Hor;
epsyAll=epsyAll.^(-0.2);
epsyAll=max(epsyAll,epsyFinal);

% Define the reward vector, depending only on the state
R = [0; 0; 0; 0; 0; 0; 10];
C=-R;   %Stick to notation in my paper

LowestQ=C(7)/(1-disc);

% step-size pars:
g = 1/(1-disc); %
rho=0.85   ;

alphaAll=1:Hor;
alphaAll=g*alphaAll.^(-rho);

alphaAll=min(alphaAll,0.1);

haA=zeros(d,d);

%Need PR averaging:
haA_PR=zeros(d,d);

PRstart=ceil(Hor/10);

% Initialize theta

thetaAll=zeros(d,Hor);
thetaNow=rand(d,1)/(1-disc);
thetaAll(:,1)= thetaNow;

Q71All=zeros(1,Hor);

state = randi(num_states); % Start in a random state
action = randi(num_actions)-1; % Start in a random input, values 0 or 1

psi=rand(d,1);

action=0;

thetaStar=[    %From 1e7 run using PR-averaging
    -379.1538;
    -379.1506;
    -379.1472;
    -379.1469;
    -379.1465;
    -379.1468;
    -535.6352;
    -801.9270;
    -700.0499;
    -619.9210;
    -703.1767;
    -699.0400;
    -992.2462;
    -233.8702];

allEpsy=logspace(-3,0,M);

allEigs=zeros(d,M);

for m = 1:M

    rng(2) %common random numbers

    m
    epsy=allEpsy(m);

    for n = 1:Hor

        alpha = alphaAll(n);

        beta=alpha^0.9;

        psi = getpsi(state,action,d);

        state_new=getX(action) ;

        Q_now=thetaStar'*psi;

        [uQ_new,u_greedy] = uQ(state_new,thetaStar,d);

        upsi_new = getpsi(state_new,u_greedy,d);

        haA=haA+ beta*( (- psi + disc*upsi_new )*psi'- haA);

        TempDiff= cost - Q_now + disc*uQ_new;

        state = state_new;

        u_greedy=0;

        action=get_action(u_greedy,epsy) ;

        psiBest = getpsi(7,0,d);
        Q71All(n)=thetaStar'*psiBest;

        if n>PRstart
            haA_PR=haA_PR+(haA-haA_PR)/(n-PRstart) ;
        end
    end
    allEigs(:,m)= eig(haA_PR);
end

figure(1)

semilogx(allEpsy, max(real(allEigs'),[],2))

figure(2)
plot(Q71All)
hold on
plot(LowestQ*ones(size(Q71All)),'--')

hold off
 
\end{verbatim}

\clearpage

 \bibliographystyle{abbrv}

\def\cprime{$'$}\def\cprime{$'$}

  \end{document}